\documentclass[twoside,11pt]{article}

\usepackage{amsmath,amsthm}
\let\standardproof\proof
\let\endstandardproof\endproof
\makeatletter
\def\input@path{{vendor/}}
\makeatother
\usepackage[preprint]{jmlr2e}
\usepackage[T1]{fontenc}
\usepackage{times}
\let\proof\standardproof
\let\endproof\endstandardproof
\hypersetup{hidelinks,pdfauthor={Zinan Wang; Xinhao Yang},
  pdftitle={Sharp Integrality Gaps in Calibration Distance}}
\newcommand{\CalDist}{\operatorname{CalDist}}
\newcommand{\LowerCalDist}{\operatorname{LowerCalDist}}

\newcommand{\R}{\mathbb{R}}

\newcommand{\eps}{\varepsilon}
\newcommand{\supp}{\operatorname{supp}}
\newcommand{\cost}{\operatorname{cost}}

\newcommand{\Modes}{\mathcal{K}}

\newcommand{\abs}[1]{\left\lvert #1\right\rvert}

\newcommand{\paperkeywords}{calibration, calibration distance, lower calibration distance, integrality gap, fractional transport}

\ShortHeadings{Sharp Integrality Gaps in Calibration Distance}{Wang and Yang}
\firstpageno{1}
\title{Sharp Integrality Gaps in Calibration Distance}
\author{\name Zinan Wang \email{zinan.wang-3@student.manchester.ac.uk}\\
        \addr University of Manchester
        \AND
        \name Xinhao Yang \email{xinhaoy@usc.edu}\\
        \addr University of Southern California}

\begin{document}
\maketitle
\begin{abstract}
We study the offline gap between deterministic calibration distance $C$
and its fractional relaxation $L$ for binary unit-weight sequences under
total absolute-change cost.  We sharpen the offline comparison
$C\leq L+O(\sqrt T)$ \citep[Theorem~2]{qiao-zheng-2024} to the sharp
worst-case order $\Theta(T^{1/3})$.  If $\Delta_T$ is the supremum of
$C-L$ over length-$T$ inputs, then
\[
 \frac{T^{1/3}}{1000}\leq\Delta_T\leq41T^{1/3}
 \qquad(T\geq216).
\]
The upper bound holds for every input, while each $T\geq216$ has a
rational lower-bound input.  For every input with $m$ distinct forecasts,
$C\leq L+m$, and the unrestricted-sample worst-case sparse order is
$\Theta(m)$.  For rational forecasts and accuracy, with binary-encoded
multiplicities of separately assignable unit identities, a grid-free
polynomial-bit-time procedure returns $B\leq L\leq U$, $U-B<\eta$, and
an exactly calibrated compact repair of cost at most $U+m\leq L+m+\eta$.

\end{abstract}

\section{Introduction}

How much can fractional splitting underestimate the cost of repairing
calibration?  We study an offline repair problem for a fixed sequence of
binary outcomes and forecasts in $[0,1]$, with one indivisible unit
identity per observation and total absolute-change cost.  A deterministic
repair assigns each identity to one destination; a fractional repair may
split its mass.  Every destination must predict its positive-label
frequency.  Write $C=\CalDist(x,p)$ and
$L=\LowerCalDist(x,p)$ for the two minimum costs, and let
$\Delta_T=\sup_{x,p}(C-L)$ over inputs of length $T$.  For a cap $m$, let
$\Delta^{\leq}_{T,m}$ denote the same supremum restricted to
$|\{p_t:t\in[T]\}|\leq m$.  The forecasts and outcomes are fixed before
repair, so this is distinct from online forecasting.

\citet[Theorem~2]{qiao-zheng-2024} established
$C\leq L+O(\sqrt T)$ in this same sum-cost model.  For $m$ distinct
original forecasts, they also proved $C\leq O(1)L+O(m)$; Appendix~A.4
gives the explicit bound $20L+8m+13$.  Section~4 asks whether the
square-root additive gap is tight and whether the sparse multiplicative
factor can be removed.  We answer both at the level of worst-case order.

Our results are as follows.  For every integer $T\geq1$, every binary
unit-weight input satisfies $C-L\leq41T^{1/3}$, whereas for each integer
$T\geq216$ a rational input satisfies $C-L\geq T^{1/3}/1000$.  Hence
\[
 \frac{T^{1/3}}{1000}\leq\Delta_T\leq41T^{1/3}
 \qquad(T\geq216),
\]
so the offline worst-case order is $\Theta(T^{1/3})$, sharpening the
established $O(\sqrt T)$ comparison.  The lower bound is existential at
each length, not a claim about every input.

The upper proof uses an adaptive rational grid $q_i=a_i/b_i$ with mesh
at most $H^{-2}$ and $\sum_i1/b_i\leq5H$; mean-preserving quantization,
denominator-sized packet rounding, and exact restoration of the positive
label count give $C-L\leq T/H^2+20H$.  Taking
$H\asymp T^{1/3}$ yields the upper order.  The lower construction uses
separated rational centers and the integer denominator of every calibrated
group mean; its certificate covers groups that mix centers and remains
valid after padding to each requested length.

For every input with $m$ distinct original forecasts,
$C\leq L+m$.  Pairing opposite-label source types, cost-nonincreasing
compression, and marginal-preserving row rounding leave at most $2m$ modes
and give the coefficient-one bound.  A separate geometric-multiplicity
construction proves $G(m)=\Theta(m)$, where $G$ is the supremum over
inputs with at most $m$ original forecasts and unrestricted $T$.  These
are not a sharp joint formula for fixed $(T,m)$.  We prove the
fixed-pair envelope $r/500\leq\Delta^{\leq}_{T,m}\leq\min\{m,41T^{1/3}\}$
when $r=\min\{\lfloor T^{1/3}\rfloor,\lfloor\sqrt{m-1}\rfloor\}\geq6$.
We also prove the all-input upper bound
$\Delta^{\leq}_{T,m}\leq36\sqrt m(1+\ln T)$, which rules out a
uniform linear-in-$\min\{m,T^{1/3}\}$ lower law.  The exact joint
order remains open.

Theorem~\ref{thm:gridfree-algorithm} gives, for rational forecasts,
rational $\eta>0$, and binary-encoded multiplicities of separately
assignable unit identities, a deterministic rational bit-model procedure
with $B\leq L\leq U$, $U-B<\eta$, and an exactly calibrated compact
repair with at most $2m$ modes and cost at most
$U+m\leq L+m+\eta$.  Its running time and output encoding size are
polynomial in the binary input and accuracy lengths; expansion can still
require $\Omega(T)$ records.  ``Grid-free'' means that no target-forecast
grid is prescribed in advance.

\subsection{Related Work}
The distributional framework of \citet{blasiok-et-al-2023} distinguishes
predictor-level distance from lower distance determined by the
prediction--label marginal.  On a uniform domain of distinct identities
with deterministic labels, their mean-cost quantities correspond to
$C/T$ and $L/T$, and their finite-target linear program uses a divisible
coupling.  Their Lemma~4.5 concerns lifts sharing a marginal with
nonuniform masses and random conditional labels.  The deterministic
four-unit separation is instead
\citet[Proposition~9]{qiao-zheng-2024}: $C=4\epsilon$ and
$L=O(\epsilon^2)$; it does not itself give a growing additive gap.

Online calibration has different quantifiers.  The
$\Omega(T^{1/3})$ lower bound of \citet{qiao-zheng-2024} is online, not a
bound on $\Delta_T$, and the deterministic predictor of
\citet{arunachaleswaran-et-al-2025} achieves $C\leq2\sqrt T+1$ against
an adversarial outcome sequence rather than comparing $C$ and $L$ for a
fixed input.

For fractional computation, \citet[Proposition~3]{hu-jambulapati-tian-yang-2024}
use a finite target grid to estimate normalized lower distance to additive
$2\epsilon$ accuracy in $O(T\log T/\epsilon^2)$ time on expanded
samples.  \citet[Corollary~32 and Appendix~C.5]{hartline-hu-wu-2026}
give a different route: their $\ell_1$-ATB is between one-third and three
times $L/T$ and is computable in $O(T\log T)$ time.  A constant-factor
surrogate is not a strict additive certificate for $L$, and sum accuracy
$\eta$ corresponds to mean accuracy $\eta/T$.  The earth-mover and
shared-label coupling characterizations of
\citet{gopalan-stavropoulos-talwar-tankala-2026} likewise concern
fractional normalized distance and do not impose integer-identity
rounding.

For the true distance, \citet{qiao-2026-hardness} already give an exact
$O(T^4)$ real-RAM algorithm on the expanded uniform, noiseless domain,
with polynomial rational overhead; this includes the expanded problem
for $C/T$.  Theorem~\ref{thm:gridfree-algorithm} instead gives strict
additive certificates for $L$ and compact calibrated output with time
polynomial in the binary input and accuracy lengths.  Binary
multiplicities represent assignable unit identities, not indivisible
weighted atoms, so hardness results for general atomic distributions do
not transfer without a separate reduction.  We claim neither the first
calibration-distance algorithm nor an improvement to the expanded
exact-computation guarantee.

Other work studies decision loss after randomized post-processing
\citep{hartline-wu-yang-2025}, testability and actionable cutoff error
\citep{rossellini-et-al-2025}, and truthfulness of online
decision-theoretic measures \citep{qiao-zhao-2025}.  These objectives are
distinct from the integrality difference for one fixed sequence.

\section{Model and Main Results}\label{sec:preliminaries}

Write $[n]=\{1,\ldots,n\}$.  An input of length $T\geq1$ consists of
outcomes $x\in\{0,1\}^T$ and forecasts $p\in[0,1]^T$.
Every sample has unit weight and retains its identity; all distances
are sums, not averages.

\begin{definition}[Deterministic calibration]
For a vector $q\in[0,1]^T$, let $I_\alpha(q)=\{t:q_t=\alpha\}$.  We call $q$ calibrated for $x$ if, for every $\alpha$ with $I_\alpha(q)\neq\varnothing$,
\[
  \sum_{t\in I_\alpha(q)}(x_t-\alpha)=0.
\]
Thus every nonempty equal-value group predicts its empirical one-frequency.
Define
\[
  \CalDist(x,p)=\min_{q\text{ calibrated for }x}\sum_{t=1}^T\abs{p_t-q_t}.
\]
\end{definition}

The minimum exists because there are finitely many partitions of $[T]$,
and each group's calibrated value is fixed by its labels.

\begin{definition}[Fractional calibration]
A finite-support fractional post-processing is a family $D=(D_t)_{t=1}^T$ of probability distributions on $[0,1]$, each with finite support.  It is fractionally calibrated if, for every destination $\alpha$,
\[
  \sum_{t=1}^T (x_t-\alpha)D_t(\alpha)=0,
  \qquad D_t(\alpha):=\Pr_{q\sim D_t}[q=\alpha].
\]
Its cost is
\[
  \cost(D;p)=\sum_{t=1}^T\sum_{\alpha\in\supp(D_t)}D_t(\alpha)\abs{p_t-\alpha},
\]
and
\[
  \LowerCalDist(x,p)=\inf_D\cost(D;p),
\]
where the infimum ranges over finite-support fractionally calibrated post-processings.
\end{definition}

Deterministic post-processing is a special case, so
\begin{equation}\label{eq:basic-order}
  0\leq \CalDist(x,p)-\LowerCalDist(x,p).
\end{equation}

For all-zero labels, calibration forces every positive-mass destination
to zero; both distances equal $\sum_t p_t$.  For all-one labels they
equal $\sum_t(1-p_t)$.  Upper-bound proofs may thus assume both labels occur.

Let $m=m(p)=|\{p_t:t\in[T]\}|$ count the distinct original forecasts,
and define the worst-case gap
\[
  \Delta_T:=\sup_{x\in\{0,1\}^T,\ p\in[0,1]^T}
  \bigl(\CalDist(x,p)-\LowerCalDist(x,p)\bigr).
\]

\begin{theorem}[Growing offline integrality gap]\label{thm:growing-lower}
For every integer $T\geq216$, there exist $x\in\{0,1\}^T$ and a rational vector $p\in[0,1]^T$ such that
\[
  \CalDist(x,p)-\LowerCalDist(x,p)\geq\frac{T^{1/3}}{1000}.
\]
Equivalently, $\Delta_T\geq T^{1/3}/1000$ for every integer $T\geq216$.
\end{theorem}

\begin{theorem}[Sparse fractional-to-deterministic comparison]\label{thm:sparse-upper}
For every finite binary unit-weight input $(x,p)$, if $p$ has $m$ distinct values, then
\[
  \CalDist(x,p)\leq \LowerCalDist(x,p)+m.
\]
This is an existence statement.  It does not assert that the proof supplies an efficient algorithm for arbitrary real-valued input.
\end{theorem}

Theorem~\ref{thm:growing-lower} is proved in Section~\ref{sec:lower-bound},
and Theorem~\ref{thm:sparse-upper} in Section~\ref{sec:sparse-rounding}.
These two comparisons concern different parameters.  The matching
length-dependent upper bound is Theorem~\ref{thm:matching-global-upper},
giving $\Delta_T=\Theta(T^{1/3})$.  The matching sparse lower bound is
Theorem~\ref{thm:sparse-order-tight}, giving $G(m)=\Theta(m)$ when $T$
is unrestricted.  The two matching constructions need not realize the
joint extremum at every fixed pair $(T,m)$.

For the capped problem there is also the all-input upper envelope
\begin{equation}\label{eq:joint-upper-summary}
 \Delta^{\leq}_{T,m}
 \leq \min\!\left\{m,\,41T^{1/3},\,36\sqrt m\,(1+\ln T)\right\}.
\end{equation}
The third term rules out a uniform lower bound of order
$\min\{m-1,T^{1/3}\}$: taking $T=n^3$ and $m=n+1$ makes its ratio
to $\min\{m-1,T^{1/3}\}=n$ tend to zero.  The exact joint order
remains open; the proof is given in Appendix~\ref{app:joint-bounds}.

\section{A Growing Lower Bound}\label{sec:lower-bound}

The construction exploits the integer denominator of every calibrated
group mean.  Fix an integer $k\geq6$ and set
\begin{equation}\label{eq:farey-set}
 \mathcal F_k=\{a/b:1\leq a<b\leq k,\ \gcd(a,b)=1,\ b>k/2\},
 \qquad K_k=|\mathcal F_k|.
\end{equation}
At each $c=a/b\in\mathcal F_k$, place $b$ unit samples at each of
$p_-=c-\eps$ and $p_+=c+\eps$, with $a$ ones and $b-a$ zeros at each
forecast, where
\begin{equation}\label{eq:farey-epsilon}
 \eps=\frac1{6k^2}.
\end{equation}
The length and number of distinct forecasts are
\begin{equation}\label{eq:Tk-mk}
 T_k=2\sum_{a/b\in\mathcal F_k}b,\qquad m_k=2K_k.
\end{equation}

\begin{lemma}[Separation and feasibility of the centers]\label{lem:center-separation}
All forecasts lie in $(0,1)$ and are distinct across clusters.
Distinct centers satisfy $|c-c'|\geq k^{-2}=6\eps$, and
$h=b\eps\leq1/(6k)<1$.
\end{lemma}
\begin{proof}
Both $c$ and $1-c$ are at least $1/k>\eps$.  For distinct reduced
fractions, $|c-c'|=|ab'-a'b|/(bb')\geq k^{-2}$, because the numerator
is a nonzero integer.  The two forecasts at one center differ by
$2\eps$, and those at different centers are at least $4\eps$ apart.
Finally $b\leq k$ gives the bound on $h$.
\end{proof}

\begin{lemma}[Finite fractional witness]\label{lem:fractional-witness}
The construction satisfies $\LowerCalDist\leq2T_k\eps^2$.
\end{lemma}
\begin{proof}
At each center, move mass $h=b\eps$ of one positive-label sample
from $p_-$ to $p_+$ and mass $h$ of one zero-label sample in the
opposite direction; leave all other mass fixed.  Both labels are available
and $0<h<1$.  Each destination retains mass $b$, while its positive
mass becomes $a-h$ or $a+h$.  Its mean is therefore
$(a\pm h)/b=c\pm\eps$, as required.  The cost is
$2h(2\eps)=4b\eps^2$ per center, hence $2T_k\eps^2$ in total.
\end{proof}

\begin{lemma}[All-group lower bound before padding]\label{lem:group-certificate}
Every nonempty deterministic calibrated group $G$, with destination
$q_G$, satisfies $\sum_{t\in G}|p_t-q_G|\geq\eps|G|$.
Consequently $\CalDist\geq\eps T_k$.
\end{lemma}
\begin{proof}
Write $N=|G|$ and $q_G=u/N$, where $u$ counts its ones.  If
$|q_G-c|\geq2\eps$ for every center, each member costs at least $\eps$.
Otherwise there is a unique center $c=a/b$ with $|q_G-c|<2\eps$.
At most $g\leq2b\leq2k$ members come from that center.  If $q_G=c$,
these members cost $\eps$ each and all other members cost at least
$5\eps$.  If $q_G\neq c$, integer separation gives
\[
 |q_G-c|=\frac{|bu-aN|}{bN}\geq\frac1{kN},
 \qquad N>3k,\qquad g<2N/3.
\]
Every member from another center then costs more than
$6\eps-\eps-2\eps=3\eps$.  There are more than $N/3$ such members,
so the group cost exceeds $\eps N$.  Summing over the partition into
equal-value groups proves the assertion.
\end{proof}

Conversely, placing all $2b$ samples of each center at $c=a/b$ is
calibrated and costs $2b\eps$.  Thus
\begin{equation}\label{eq:exact-C-family}
 \CalDist=\eps T_k.
\end{equation}

\begin{lemma}[Farey counting bounds]\label{lem:farey-count}
For every integer $k\geq6$,
\[
 \frac{k^2}{72}\leq K_k\leq\frac{k^2}{2},\qquad
 \frac{k^3}{72}\leq T_k\leq k^3,\qquad
 \frac{k^2}{36}\leq m_k\leq k^2.
\]
\end{lemma}
\begin{proof}
The number of noncoprime pairs in $[k]^2$ is at most
$k^2\sum_{d=2}^\infty d^{-2}\leq25k^2/36$, using
$1/4+1/9+\int_3^\infty x^{-2}dx=25/36$.
There are therefore at least $11k^2/36$ coprime ordered pairs.
Remove the sole coprime diagonal pair $(1,1)$, divide by two, and
remove at most $k^2/8$ pairs with $a<b\leq\lfloor k/2\rfloor$:
\[
 K_k\geq\frac{11k^2}{72}-\frac12-\frac{k^2}{8}
       =\frac{k^2}{36}-\frac12\geq\frac{k^2}{72}.
\]
The upper bound follows by counting all $a<b\leq k$.
Since $k/2<b\leq k$, we have $kK_k\leq T_k\leq2kK_k$;
the remaining bounds follow from \eqref{eq:Tk-mk}.
\end{proof}
Combining these lemmas yields
\begin{equation}\label{eq:family-gap}
 \CalDist-\LowerCalDist\geq T_k(\eps-2\eps^2)
 \geq\frac{k}{432}\left(1-\frac1{3k^2}\right)
 \geq\frac{k}{500}\geq\frac{\sqrt{m_k}}{500}.
\end{equation}

\subsection{Padding to Every Length}\label{subsec:padding}
For $T\geq216$, take $k=\lfloor T^{1/3}\rfloor\geq6$ and append
$T-T_k\geq0$ samples with $(x_t,p_t)=(0,0)$.  Leaving their mass at
zero extends the fractional witness at no cost, so
\begin{equation}\label{eq:padded-L}
 \LowerCalDist\leq2T_k\eps^2.
\end{equation}
We must also allow the appended samples in arbitrary calibrated groups.
In such a group let $N$ be its total size, $q_G=u/N$ its mean, and
$r$ its number of original samples.  A group with $r=0$ has nonnegative
cost.  If $q_G$ is at least $2\eps$ from every center, all $r$
original samples cost at least $\eps$ each.  If $q_G=c$, the same
holds by separation.  In the remaining case $0<|q_G-c|<2\eps$,
the preceding integer argument still gives $N>3k$ and at most $2k$
members from $c$.  Original members from other centers cost more than
$3\eps$; each appended member costs
$q_G\geq c-2\eps\geq1/k-2\eps\geq3\eps$.
More than $N/3$ members therefore cost at least $3\eps$, and the
group cost exceeds $\eps N\geq\eps r$.  Summation gives
\begin{equation}\label{eq:padded-C}
 \CalDist\geq\eps T_k.
\end{equation}
Finally $k\geq T^{1/3}/2$ and $1-1/(3k^2)\geq107/108$, whence
\begin{equation}\label{eq:all-T-gap}
 \CalDist-\LowerCalDist
 \geq\frac{k}{432}\left(1-\frac1{3k^2}\right)
 \geq\frac{107}{93312}T^{1/3}>\frac{T^{1/3}}{1000}.
\end{equation}
This proves Theorem~\ref{thm:growing-lower}, without a monotonicity
assumption on deterministic cost under padding.

\section{A Sparse Fractional-to-Deterministic Upper Bound}\label{sec:sparse-rounding}

Write the distinct forecasts as $s_1,\ldots,s_m$, and let
$n_{i,y}=|\{t:p_t=s_i,\ x_t=y\}|$ for each source type $(i,y)$.
We pair types at each fractional destination, compress the modes to
$2m$, and round each source row before recalibrating.

\begin{lemma}[Paired-mode representation]\label{lem:paired-modes}
Suppose both labels occur.  Every finite fractional witness of cost $F$
has a finite mode representation $(q_k,\lambda_k,a_k)_{k\in\Modes}$,
where $q_k\in[0,1]$, $\lambda_k\geq0$, and $a_k\in\R^{2m}$ has
at most one zero-label and one one-label coordinate, of values
$1-q_k$ and $q_k$.  Moreover,
\[
 \sum_k\lambda_ka_k=n,\quad \sum_k\lambda_k=T,\quad
 \sum_k\lambda_kc_k=F,\qquad
 c_k=(1-q_k)|s_i-q_k|+q_k|s_j-q_k|,
\]
where $(i,0),(j,1)$ are the source types of mode $k$.
\end{lemma}
\begin{proof}
Aggregate the mass sent to destination $q$ by zero-label type as
$u_{i,q}$ and one-label type as $v_{j,q}$, and put
$M_q=\sum_i u_{i,q}+\sum_jv_{j,q}$.  Discard $M_q=0$.
Calibration says
\begin{equation}\label{eq:aggregate-calibration}
 \sum_i u_{i,q}=(1-q)M_q,\qquad \sum_jv_{j,q}=qM_q.
\end{equation}
For $0<q<1$, use the mode weights
\begin{equation}\label{eq:pair-weight}
 \lambda_{ij,q}=\frac{u_{i,q}v_{j,q}}{q(1-q)M_q}.
\end{equation}
Then $\sum_j(1-q)\lambda_{ij,q}=u_{i,q}$,
$\sum_iq\lambda_{ij,q}=v_{j,q}$ and
$\sum_{i,j}\lambda_{ij,q}=M_q$.  These identities preserve source
mass, destination mass and transport cost.  At $q=0$ only zero-label
mass occurs; use one mode of weight $u_{i,0}$ per type and any existing
one-label type as its zero-coefficient partner.  Treat $q=1$ symmetrically.
Summing over the finite destination set proves all assertions.
\end{proof}

\subsection{Compression}
For these modes consider the finite linear program
\begin{equation}\label{eq:mode-lp}
 \min_{\theta\geq0}\sum_k\theta_kc_k
 \quad\text{subject to}\quad\sum_k\theta_ka_k=n.
\end{equation}
It is feasible, and every feasible point has $\sum_k\theta_k=T$.
Its feasible region is consequently compact.
\begin{lemma}[Cost-nonincreasing compression]\label{lem:mode-compression}
There is an optimum $\theta^*$ of \eqref{eq:mode-lp} with
$J=|\{k:\theta_k^*>0\}|\leq2m$ and cost $F'\leq F$.
\end{lemma}
\begin{proof}
Choose an optimum with smallest positive support.  If its columns are
dependent, choose nonzero $h$ on that support with $\sum_kh_ka_k=0$.
Column sums imply $\sum_kh_k=0$, so $h$ has both signs.  Orient it so
$\sum_kh_kc_k\leq0$ and increase $\tau>0$ in $\theta^*+\tau h$ until
one coordinate vanishes.  A negative cost derivative contradicts
optimality; a zero derivative contradicts minimal support.  Thus the
positive columns are independent and number at most $2m$.
\end{proof}
The compressed representation is a genuine fractional witness: set
$w_{r,k}=\theta_k^*(a_k)_r$, and give each sample of every nonempty
type $r$ the distribution $D'_t(q)=n_r^{-1}\sum_{k:q_k=q}w_{r,k}$.
Row sums are $n_r$, each mode is calibrated, and the cost is $F'$.
Empty types have only zero entries and require no division.

\subsection{Rounding and Recalibration}
Apply Lemma~\ref{lem:row-rounding} independently to the source rows.
The resulting finite random integer array satisfies
\begin{equation}\label{eq:rounding-properties}
 \sum_kN_{r,k}=n_r,\qquad \mathbb EN_{r,k}=w_{r,k},\qquad
 w_{r,k}=0\Longrightarrow N_{r,k}=0.
\end{equation}
Assign actual identities of type $r$ to these integer slots.
A mode contains at most one type per label; write its label counts as
$z_k,o_k$, with an absent row interpreted as zero.  For $z_k+o_k>0$,
give the mode destination $q'_k=o_k/(z_k+o_k)$.  Discard empty modes;
coincident destinations may be merged because each mode is calibrated.
Define its residual cost for every realization by
\begin{equation}\label{eq:mode-residual-definition}
 R_k=\begin{cases}
 (z_k+o_k)|q'_k-q_k|,&z_k+o_k>0,\\
 0,&z_k+o_k=0.
 \end{cases}
\end{equation}
Put $Z_k=(1-q_k)\theta_k^*$, $O_k=q_k\theta_k^*$ and
$\delta_{0,k}=z_k-Z_k$, $\delta_{1,k}=o_k-O_k$.
For a nonempty mode,
\begin{equation}\label{eq:recalibration-cost}
 R_k=|(1-q_k)o_k-q_kz_k|
     =|(1-q_k)\delta_{1,k}-q_k\delta_{0,k}|.
\end{equation}
The empty-mode definition therefore also satisfies
\begin{equation}\label{eq:recalibration-cost-all}
 R_k\leq(1-q_k)|\delta_{1,k}|+q_k|\delta_{0,k}|.
\end{equation}
If $f_{r,k}$ is the fractional part of $w_{r,k}$, row rounding gives
$\mathbb E|N_{r,k}-w_{r,k}|=2f_{r,k}(1-f_{r,k})\leq1/2$.
There is at most one row per label in each mode, hence
\begin{equation}\label{eq:expected-mode-residual}
 \mathbb ER_k\leq(1-q_k)/2+q_k/2=1/2.
\end{equation}
This includes endpoint modes: their absent opposite-label row stays
zero, so their nonempty realizations retain the endpoint destination.
The triangle inequality, row marginals and $J\leq2m$ now yield
\begin{equation}\label{eq:expected-final-cost}
 \mathbb E\sum_t|p_t-q'_t|
 \leq\mathbb E\sum_{r,k}N_{r,k}|s_i-q_k|+\sum_k\mathbb ER_k
 \leq F'+J/2\leq F+m.
\end{equation}
Some realization of this finite random construction costs at most $F+m$.
For every $\eta>0$, take a finite witness with
$F<\LowerCalDist+\eta$ and let $\eta\downarrow0$.
This proves Theorem~\ref{thm:sparse-upper}.  Pure-label inputs have
$\CalDist=\LowerCalDist$ by Section~\ref{sec:preliminaries}.
No fractional optimum or finite encoding of real forecasts was assumed.

\section{A Grid-Free Rational Algorithmic Reduction}
\label{sec:gridfree-algorithm}

This section gives the rational bit-model algorithmic consequence of the
finite binary model.  Multiplicities are binary encoded, so the input is
compressed but the underlying observations remain separately assignable unit
identities.  All optimization and repair operations below are exact rational
operations; the only external complexity parameters are those of the standard
GLS weak-optimization reduction~\citep{groetschel-lovasz-schrijver-1988}.

\subsection{Input Model and Statement}

Let $a_1,\ldots,a_m\in[0,1]$ be the distinct rational forecasts, and let
$n_i^0,n_i^1\in\mathbb Z_{\geq0}$ be the numbers of unit observations with
forecast $a_i$ and label $0$ or $1$.  Put
\[
 T:=\sum_{i=1}^m(n_i^0+n_i^1).
\]
The counts are binary encoded and denote separately assignable identities.
For a self-delimiting rational encoding, let $N_\eta$ be the length of
$(\eta,(a_i,n_i^0,n_i^1)_{i=1}^m)$; thus
\begin{equation}
 N_\eta=\operatorname{bits}(\eta)+\sum_{i=1}^m
 \bigl(\operatorname{bits}(a_i)+\operatorname{bits}(n_i^0)
       +\operatorname{bits}(n_i^1)\bigr)+O(m\log m).
 \label{eq:alg-input-length}
\end{equation}
The expanded sequence has $n_i^0$ copies of $(x,p)=(0,a_i)$ and $n_i^1$
copies of $(1,a_i)$, in any fixed order, and defines
$L=\LowerCalDist(x,p)$.  Delete zero-count rows and retain the same notation
for the remaining rows.  Write $\eta_0:=\min\{\eta,1\}>0$.

\begin{theorem}[Grid-free rational reduction]
\label{thm:gridfree-algorithm}
For every compressed binary unit-weight input with $T\geq1$, rational
forecasts $a_i\in[0,1]$, and rational $\eta>0$, a deterministic rational
bit-model procedure returns rational certificates $B,U$ and a compact integer
mode list such that
\begin{align}
 B&\leq L:=\LowerCalDist(x,p)\leq U, & U-B&<\eta,
 \label{eq:alg-cert-interval}\\
 C_{\mathrm{out}}&\leq U+m\leq L+m+\eta. &&
 \label{eq:alg-output-bound}
\end{align}
The list is exactly calibrated, uses every source-row count, and has at most
$2m$ nonempty modes.  Both the running time and output encoding size are
bounded by a fixed polynomial in
$N_\eta,m,\log(T+1),\log(1+1/\eta_0)$; explicitly expanding it costs
$\Omega(T)$ output records.  The procedure is a GLS rational
weak-separation-to-weak-optimization reduction, not an exact algorithm for
$L$ or $\CalDist$.  The full-vector bound $P_{\rm sep}$, body bound
$P_{\rm body}$, and local bit-length recurrence in
Lemma~\ref{lem:alg-oracle-encoding} compose with the oracle-answer
hypothesis of GLS~\S1.2 and Corollary~4.2.7; the latter's fixed
external coefficients remain symbolic.
\end{theorem}

The proof has five ingredients: destination pairing, a finite convex
representation, exact rational certificate repair, support compression, and
legal deterministic row rounding.  The two pure-label cases are explicit
endpoint assignments.

\subsection{Finite Paired Representation}

Assume both labels occur and put
\[
 I_0:=\{i:n_i^0>0\},\qquad I_1:=\{j:n_j^1>0\}.
\]
For $i\in I_0$, $j\in I_1$, and $q\in[0,1]$, define
\begin{equation}
 c_{ij}(q):=(1-q)|a_i-q|+q|a_j-q|.
 \label{eq:alg-pair-cost}
\end{equation}

\begin{lemma}[Destination pairing]
\label{lem:alg-pairing}
Every finite fractional repair can be represented by paired modes with the
same source masses and cost.  At an interior destination $q$, if $U_i(q)$ and
$V_j(q)$ are the transported zero- and one-label masses and
$H_q:=\sum_iU_i(q)+\sum_jV_j(q)$, use
\begin{equation}
 \lambda_{ijq}:=\frac{U_i(q)V_j(q)}{q(1-q)H_q}.
 \label{eq:alg-product-pairing}
\end{equation}
At $q=0$ or $1$, the absent label is only a zero-mass partner.
\end{lemma}

\begin{proof}
Discard destinations with $H_q=0$.  Calibration gives $\sum_iU_i(q)=(1-q)H_q$ and
$\sum_jV_j(q)=qH_q$, so summing the displayed product weights over one
index recovers the corresponding source mass.  Linearity preserves cost.
At $q=0$, all $V_j(0)$ vanish and the positive $U_i(0)$ can use one fixed
one-row partner; $q=1$ is symmetric.  Thus no endpoint division is made.
\end{proof}

For $a,b\in[0,1]$, let $\overline c_{ab}$ be the largest convex minorant of
$c_{ab}$.  Direct expansion gives, for $a\leq b$,
\begin{equation}
\overline c_{ab}(q)=
\begin{cases}
a+(b-a-1)q,&0\leq q\leq a,\\
2ab-a+(1-a-b)q,&a\leq q\leq b,\\
-a+(1+a-b)q,&b\leq q\leq1.
\end{cases}
\label{eq:alg-convex-envelope}
\end{equation}
Indeed $c_{ab}$ exceeds its middle endpoint chord by
$2(q-a)(b-q)$, and the three slopes are nondecreasing.  If $a>b$, the
middle quadratic and the joins make $c_{ab}$ convex, so
$\overline c_{ab}=c_{ab}$.  Define its perspective by
\begin{equation}
 g_{ab}(u,v):=
 \begin{cases}(u+v)\overline c_{ab}(v/(u+v)),&u+v>0,\\0,&u=v=0.
 \end{cases}
 \label{eq:alg-perspective}
\end{equation}

\begin{lemma}[Finite convex representation]
\label{lem:alg-representation}
Writing $L=\LowerCalDist(x,p)$,
\begin{equation}
L=\min\left\{\sum_{i\in I_0,j\in I_1}g_{a_i a_j}(u_{ij},v_{ij}):
\begin{array}{l}
u_{ij},v_{ij}\geq0,\\
\sum_j u_{ij}=n_i^0\ (i\in I_0),\quad
\sum_i v_{ij}=n_j^1\ (j\in I_1)
\end{array}\right\}.
\label{eq:alg-finite-representation}
\end{equation}
The minimum is attained, and every feasible point lifts to a finite calibrated
transport of the same cost with at most two destinations per source pair.
\end{lemma}

\begin{proof}
For paired modes from Lemma~\ref{lem:alg-pairing}, aggregate
$u_{ij}=\sum_q(1-q)\lambda_{ijq}$ and
$v_{ij}=\sum_q q\lambda_{ijq}$.  Jensen's inequality for the convex minorant
gives the displayed objective as a lower bound.  Conversely, for
$\theta=u+v>0$, set $q=v/\theta$.  Use one mode at $q$ if $a\geq b$ or
$q\notin(a,b)$.  If $a<b$ and $a<q<b$, use endpoints $a,b$ with weights
\[
 \theta_a=\theta\frac{b-q}{b-a},\qquad
 \theta_b=\theta\frac{q-a}{b-a}.
\]
They preserve $u,v$ and attain the chord $g_{ab}(u,v)$.  Divide incident
masses by each positive integer source count and assign the resulting
probability vector to every identity of that source type.  The feasible set is
compact and $g$ is continuous at the origin, so the minimum is attained.
If $I_1=\varnothing$, send all zero-label identities to $0$ and both distances
equal $\sum_t p_t$; if $I_0=\varnothing$, send all one-label identities to
$1$ and both equal $\sum_t(1-p_t)$.
\end{proof}

\subsection{Rational Primal and Dual Bodies}

For $a>b$, put
\begin{equation}
 d_{ab}(u,v):=au-(1-b)v,\qquad
 \phi_{ab}(u,v):=au-(1+b)v+\frac{2v^2}{u+v}\quad(u+v>0).
 \label{eq:alg-perspective-quadratic}
\end{equation}
Then
\begin{equation}
 g_{ab}(u,v)=
 \begin{cases}\max\{d_{ab}(u,v),-d_{ab}(u,v),\phi_{ab}(u,v)\},&u+v>0,\\0,&u=v=0,
 \end{cases}
 \label{eq:alg-perspective-ab}
\end{equation}
and the quadratic piece is active exactly when $b\leq v/(u+v)\leq a$.
With $\theta=u+v$ and $w=z-au+(1+b)v$, its epigraph is exactly
\begin{equation}
 z\geq d_{ab}(u,v),\quad z\geq-d_{ab}(u,v),\quad
 \theta\geq0,\quad w\geq0,\quad \theta w\geq2v^2.
 \label{eq:alg-rotated-cone}
\end{equation}
For $a\leq b$, define
\begin{align}
 A_{ab}^{(1)}(u,v)&:=au+(b-1)v,\\
 A_{ab}^{(2)}(u,v)&:=(2ab-a)u+(2ab+1-2a-b)v,\\
 A_{ab}^{(3)}(u,v)&:=-au+(1-b)v.
 \label{eq:alg-affine-perspectives}
\end{align}
On the whole nonnegative quadrant, including the origin,
\begin{equation}
 g_{ab}(u,v)=\max\{A_{ab}^{(1)}(u,v),A_{ab}^{(2)}(u,v),A_{ab}^{(3)}(u,v)\}.
 \label{eq:alg-perspective-affine}
\end{equation}

\begin{lemma}[Primal and dual separation]
\label{lem:alg-bodies}
Eliminate one $u$ coordinate per active zero row and one $v$ coordinate per
active one row.  With $0\leq h\leq T+1$, the primal epigraph is a bounded
full-dimensional rational body $K_P$.  A uniform positive flow and
$h^0=T+1/2$ give the inner radius $r_P:=1/(100m^3)$, and
$R_P:=4m(T+1)$ is an outer radius.  In active coordinates the dual value is
\begin{equation}
 L=\max\left\{\sum_{i\in I_0}n_i^0\alpha_i+
 \sum_{j\in I_1}n_j^1\beta_j:
 (1-q)\alpha_i+q\beta_j\leq c_{ij}(q)\ \forall i,j,q\right\}.
 \label{eq:alg-dual}
\end{equation}
The dual body is the full-dimensional active product body
\begin{equation}
 K_D:=\left\{(\alpha,\beta)\in[-2T,1]^{I_0}\times[-2T,1]^{I_1}:
 (1-q)\alpha_i+q\beta_j\leq c_{ij}(q)\ \forall i,j,q\right\},
 \label{eq:alg-active-dual-body}
\end{equation}
in $\mathbb R^{I_0}\times\mathbb R^{I_1}$, of dimension
$d_D=|I_0|+|I_1|$.  It has inner radius $\rho_D:=1/2$ and outer radius
$R_D:=4T(m+1)$.  The exact strong separation oracles and their complete
encoding bounds are rational and are accounted for by
Lemma~\ref{lem:alg-oracle-encoding} and Appendix~\ref{app:algorithm-encoding}.
\end{lemma}

\begin{proof}
The displayed affine gradients and the quadratic gradient below have
coordinate magnitudes at most $2$ on the nonnegative quadrant.  Thus
\begin{equation}
 |g_{ab}(u,v)-g_{ab}(u',v')|
 \leq5(|u-u'|+|v-v'|).
 \label{eq:alg-g-lipschitz}
\end{equation}
The uniform feasible flow has objective at most $T$ and entries at least
$1/m$.  Its reconstructed perturbation bounds, both inner balls, and
strong duality are proved in Appendix~\ref{app:algorithm-geometry}.

For completeness, the primal oracle first checks nonnegativity, row sums, and
$h$ bounds.  For $a>b$, at $\theta=u+v>0$ the complete quadratic gradient is
\begin{equation}
 \left(a-2(v/\theta)^2,
 -(1+b)+4v/\theta-2(v/\theta)^2\right),
 \qquad \theta=u+v;
 \label{eq:alg-quadratic-gradient}
\end{equation}
the affine pieces use their displayed constant gradients.  At $\theta=0$ use
the zero subgradient.  Select a maximizing candidate for every pair and sum
the supporting inequalities.  If the query $(u^0,v^0,h^0)$ violates the
epigraph, the single pulled-back aggregate cut is
\[
 h\geq F(u^0,v^0)+\sum_{i,j}\langle s_{ij},
 (u_{ij},v_{ij})-(u^0_{ij},v^0_{ij})\rangle,
 \qquad F:=\sum_{i,j}g_{a_i a_j},
\]
after substituting all eliminated row and column coordinates.  It is valid
for $K_P$ and violated at the query.  The rotated-cone test in
\eqref{eq:alg-rotated-cone} is exact rational arithmetic.

For the dual, the oracle checks every coordinate of the auxiliary box first:
it returns $-z_\ell\leq2T$ if $y_\ell<-2T$ and $z_\ell\leq1$ if
$y_\ell>1$.  Only after this box check it scans the continuous slack
$s_{ij}(q)=c_{ij}(q)-(1-q)\alpha_i-q\beta_j$.  On each interval cut by
$0,a_i,a_j,1$, the slack is affine or quadratic, so its exact minimum is at an
endpoint or a rational stationary point.  A negative minimum returns the
two-coordinate constraint cut with normal divided by its largest absolute
coordinate; otherwise membership is exact.
At $q=0,1$, feasibility gives $\alpha_i\leq a_i$ and
$\beta_j\leq1-a_j$.  The supporting-potential argument in
Appendix~\ref{app:algorithm-geometry} shows that $[-2T,1]$ retains a
dual optimizer, and that deleted-row potentials can be extended by $-1$.
The box tests and slack scan together therefore constitute exact separation
of the stated active body.
\end{proof}

\subsection{Weak Optimization and Exact Certificates}

We use the following GLS interface for both bodies.
\begin{lemma}[GLS weak-optimization interface and error transfer]
\label{lem:alg-gls-interface}
Let $K\subseteq\mathbb R^d$ be a convex body with rational description,
$K\subseteq B_2(0,R)$ and $B_2(a_0,r)\subseteq K$ for rational
$a_0,r,R$.  Given a rational objective $c\in\mathbb Q^d$ and
rational $0<\tau<r$, set
\begin{equation}
 \begin{gathered}
 \|z\|_2=(\sum_jz_j^2)^{1/2},\qquad \|z\|_\infty=\max_j|z_j|,\\
 S(K,\sigma)=\{z:\operatorname{dist}_2(z,K)\leq\sigma\},\qquad
 S(K,-\sigma)=\{z:B_2(z,\sigma)\subseteq K\}.
 \end{gathered}
 \label{eq:alg-gls-norms}
\end{equation}
The GLS body input is
\begin{equation}
 (K;d,R,\mathcal O_{\rm WSEP}),\qquad K\subseteq B_2(0,R),\qquad0<\tau<r.
 \label{eq:alg-gls-interface-body}
\end{equation}
A rational WSEP query $(z,\sigma)$, $\sigma>0$, returns membership in $S(K,\sigma)$ or a rational
$v$ with $\|v\|_\infty=1$ such that
\begin{equation}
 v^\mathsf Tx<v^\mathsf Tz+\sigma\quad(x\in S(K,-\sigma)).
 \label{eq:alg-gls-wsep-output}
\end{equation}
WOPT returns rational $y\in S(K,\tau)$ satisfying
\begin{equation}
 c^\mathsf Tx<c^\mathsf Ty+\tau\quad(x\in S(K,-\tau)),
 \label{eq:alg-gls-wopt-output}
\end{equation}
or reports an empty interior kernel; the latter is impossible here.  If
$\|c\|_2\leq M$ and $x^*$ maximizes $c^\mathsf Tx$ over $K$, then
\begin{equation}
 x^*_{\tau}:=(1-\tau/r)x^*+(\tau/r)a_0\in S(K,-\tau),\qquad
 c^\mathsf Tx^*-c^\mathsf Tx^*_{\tau}\leq2MR\tau/r.
 \label{eq:alg-gls-interior-compensation}
\end{equation}
Consequently,
\begin{equation}
 \max_{x\in K}c^\mathsf Tx-c^\mathsf Ty
 \leq\tau(1+2MR/r),
 \label{eq:alg-gls-loss}
\end{equation}
and the rational choice
\begin{equation}
 \tau=\frac{\varepsilon r}{2(r+2MR)}
 \label{eq:alg-gls-delta}
\end{equation}
loses at most $\varepsilon/2$.  For the primal and dual calls this gives
\begin{equation}
 \delta_P=\frac{\varepsilon_P r_P}{2(r_P+2R_P)},\qquad
 \delta_D=\frac{\varepsilon_D\rho_D}{2(\rho_D+2TR_D)},
 \quad \varepsilon_P=\frac{\eta_0}{100m^3},\quad
 \varepsilon_D=\frac{\eta_0}{4(T+1)}.
 \label{eq:alg-gls-body-deltas}
\end{equation}
\end{lemma}

\begin{proof}
For $\|e\|_2\leq\tau$, write
$x^*_\tau+e=(1-\tau/r)x^*+(\tau/r)(a_0+re/\tau)\in K$.
Thus $x^*_\tau$ lies in the interior kernel.  The objective bounds follow
from $\|x^*-a_0\|_2\leq2R$ and the WOPT guarantee.  A strong exact
cut is also a valid WSEP answer for every signed neighborhood.  This is the
only GLS error-transfer argument used in the proof.
\end{proof}

Put
\begin{equation}
 S:=N_\eta+m+\lceil\log_2(T+1)\rceil+
 \lceil\log_2(1/\eta_0)\rceil+1.
 \label{eq:alg-composition-size}
\end{equation}

\begin{lemma}[Fraction-free kernel-bit bound]
\label{lem:alg-kernel-bits}
For $r$ active source coordinates, $r\leq2m$, select $r+1$ positive
paired source columns.  An exact integer kernel direction, including
rank-deficient cases, and the common-denominator pivot update can be computed with
the integer-payload bounds $K_{2m}(h_0)$ and $R_{\rm B}(h_0)$ of
Appendix~\ref{app:algorithm-kernel}; in particular
$K_{2m}(h_0)\leq(2m+2)h_{\rm ker}$.  The direction preserves the full active
source vector in $\mathbb R^{I_0}\times\mathbb R^{I_1}$ and has both signs.
\end{lemma}

\begin{proof}
Columnwise denominator clearing, fraction-free elimination, exact
back-substitution, and the ratio pivot are proved in full in
Appendix~\ref{app:algorithm-kernel}; rank-deficient selected columns are
handled there by selecting independent rows and extending by zero.
\end{proof}

\begin{lemma}[Fixed-polynomial oracle encoding bound]
\label{lem:alg-oracle-encoding}
For $S$ in \eqref{eq:alg-composition-size}, the exact primal and dual oracles
have answer length at most $P_{\rm sep}(S,Q)$ and time
$\operatorname{Time}_{\rm oracle}(S,Q)$ for a query of length $Q$; body
descriptions have length at most $P_{\rm body}(S)$.  The two GLS calls use
\[
 N_*\leq2P_{\rm calls}(\Lambda_S),\qquad
 Q_{\max}\leq P_{\rm qry}(\Lambda_S),
\]
and all post-GLS operands have length at most $P_{\rm int}(S)$.  The local
step bound is $P_{\rm steps}(S)$.  The locally explicit polynomials
$P_{\rm sep},P_{\rm body},P_{\rm int}$, their composition with the
fixed external GLS polynomials, and $P_{\rm steps}$ are defined in
Appendix~\ref{app:algorithm-encoding}.  Thus
\begin{equation}
\begin{split}
 \operatorname{Time}_{\rm compact}
 &\leq P_{\rm setup}(S)+2P_{\rm ext}(\Lambda_S)
       +N_*[Q_{\max}+P_{\rm sep}(S,Q_{\max})+1]\\
 &\quad+N_*\operatorname{Time}_{\rm oracle}(S,Q_{\max})
       +A_{\rm post}(S)\beta_{\rm bit}(P_{\rm int}(S))\\
 &\leq P_{\rm steps}(S).
\end{split}
\label{eq:alg-total-time-main}
\end{equation}
The external GLS polynomials have symbolic coefficients
$c_C,c_Q,c_W$ and degrees $d_C,d_Q,d_W$ from the cited oracle-machine
theorem.  No numerical GLS constants are asserted.
The parameter $\Lambda_S$ is defined in \eqref{eq:alg-gls-composed-bound}.
$P_{\rm setup}$ and $\beta_{\rm bit}$ are the one-time construction bound
and reference-machine primitive charge proved in Appendix~\ref{app:algorithm-encoding}.
\end{lemma}

\begin{proof}
The displayed bound includes body setup and the exact oracle ledger,
the two external GLS transcripts, query/answer communication, and the
post-processing closure.  Its fixed local factors, including all cut-vector
coordinates, kernel registers, denominator recurrences, and rational-record
operations, are proved in Appendix~\ref{app:algorithm-bits}.
\end{proof}

\begin{lemma}[Exact rational certificate extraction]
\label{lem:alg-certificates}
The two WOPT outputs can be repaired to exact rational certificates satisfying
\begin{equation}
 B\leq L\leq U,\qquad U-B<\eta.
 \label{eq:alg-certificate-width}
\end{equation}
The upper certificate is a finite calibrated transport and the lower
certificate satisfies every dual inequality for every $q\in[0,1]$.
\end{lemma}

\begin{proof}
Run the primal call with $c_P=-e_h$ and the dual call with
$c_D=((n_i^0)_{i\in I_0},(n_j^1)_{j\in I_1})$.
The interface gives a rational point $\bar y_P$ within
$\delta_P\leq\varepsilon_P$ of $K_P$, with
$h(\bar y_P)<L+\varepsilon_P/2$.  Apply the explicit rational repair map
of Lemma~\ref{lem:alg-primal-repair-map}: reconstruct the full flow,
truncate negative entries, normalize each active row to its prescribed
integer mass, recompute the epigraph value exactly, and then apply the
$Q$-grid largest-remainder map.  No feasible point is selected by an oracle.
The lemma proves that all denominators are positive, that the repaired flow
is feasible, and that its cost before the $Q$-grid step is below
$L+21\eta_0/50$.  With
$Q=\lceil40m^2/\eta_0\rceil$, floor $Q$ times each row and ceil only positive
fractional entries until its exact row sum is restored.  This changes each
entry by less than $1/Q$.  There are at most $2m^2$ entries, so the
additional cost is at most $10m^2/Q\leq\eta_0/4$.
The lifted finite transport consequently has exact cost
\begin{equation}
 L\leq U<L+\frac{21}{50}\eta_0+\frac14\eta_0
   =L+\frac{67}{100}\eta_0.
 \label{eq:alg-upper-certificate}
\end{equation}

The dual objective has Euclidean norm at most $T$, so the interface gives
$c_D^{\mathsf T}\bar y_D\geq L-\varepsilon_D/2$.
For this output, scan the exact rational minimum of every continuous
slack, even if the near-feasible point lies outside the auxiliary box.
Set $\gamma$ to zero if this minimum is nonnegative and to its opposite
otherwise, and subtract $\gamma$ from every potential.  The
normal of each constraint has Euclidean norm at most one, so
$0\leq\gamma\leq\delta_D\leq\varepsilon_D$.  The shifted vector is an
exact feasible certificate, possibly outside the auxiliary box.  Its
objective falls by $T\gamma$, so
\begin{equation}
 B\geq L-(T+1)\varepsilon_D=L-\frac14\eta_0.
 \label{eq:alg-lower-certificate}
\end{equation}
Weak duality gives $B\leq L$, hence
$0\leq U-B<(67/100+1/4)\eta_0
 =\frac{23}{25}\eta_0<\eta$.  The exact box check precedes the
dual slack scan during GLS, but the final repair always scans the full
$q\in[0,1]$ interval.

For the final record bounds, let
\begin{equation}
 Q_*:=P_{\rm int}(S),\qquad d_P\leq2m^2+1,\qquad d_D\leq2m,
 \label{eq:alg-final-object-length}
\end{equation}
and let $N_*$ be bounded as in Lemma~\ref{lem:alg-oracle-encoding}.  Every
transcript, repair, pivot, rounding state, and compact record has length at
most $Q_*$, and the exact local time is
\begin{equation}
 \operatorname{Time}_{\rm compact}
 \leq P_{\rm setup}(S)+\mathcal T_{\rm GLS}(S)
          +A_{\rm post}(S)\beta_{\rm bit}(Q_*).
 \label{eq:alg-final-bit-time}
\end{equation}
The external quantities in this display are precisely the symbolic GLS
quantities bounded in Appendix~\ref{app:algorithm-encoding}.
\end{proof}

\subsection{Support Compression and Deterministic Integerization}

The finite lift of the upper certificate has at most $2m^2$ paired modes.
A positive mode $k$ has masses $u_k,v_k$, weight $w_k=u_k+v_k$, destination
$q_k=v_k/w_k$, source column $A_k$ with entries $1-q_k,q_k$ at its two source
rows, and cost $c_k$.  Thus $\sum_kA_kw_k=n$ and $F=\sum_kc_kw_k=U$.

\begin{lemma}[Exact mode compression]
\label{lem:alg-compression}
There is a rational cost-nonincreasing replacement with at most
$J\leq r\leq2m$ positive modes, where $r=|I_0|+|I_1|$ is the number of
active source rows.  Starting from at most $2m^2$ modes, all intermediate
records have polynomial bit length in the compact input and certificate
lengths.
\end{lemma}

\begin{proof}
If more than $r$ modes remain, choose $r+1$ columns and a nonzero rational
kernel vector $d$ with $Ad=0$.  Since every column sums to one, $\sum_kd_k=0$;
orient it so $\sum_kc_kd_k\leq0$, and set
\[
 t:=\min_{d_k<0}\frac{w_k}{-d_k},\qquad w_k\leftarrow w_k+td_k.
\]
Row totals and nonnegativity are preserved, cost does not increase, and a
mode disappears.  Repeat at most $2m^2-r$ times.  Exact rank-deficient
kernel construction, sign orientation, ratio comparison, and denominator
growth are those of Lemma~\ref{lem:alg-kernel-bits} and
Appendix~\ref{app:algorithm-denominators}.
\end{proof}

\begin{lemma}[Deterministic row rounding]
\label{lem:alg-rounding}
For $J\leq2m$ paired rational modes of cost $F'$, there is a deterministic
integer assignment preserving every source row and calibrating every nonempty
mode such that
\begin{equation}
 C_{\mathrm{out}}\leq F'+\frac J2.
 \label{eq:alg-rounding-bound}
\end{equation}
\end{lemma}

\begin{proof}
Write each incident mass as $u_k=\ell_{0k}+f_{0k}$ or
$v_k=\ell_{1k}+f_{1k}$, with $0\leq f_{yk}<1$.  For source row $r$, the
positive fractional support and legal number of ceilings are
\begin{equation}
 S_r:=\{k:f_{rk}>0\},\qquad R_r:=\sum_{k\in S_r}f_{rk}\in\mathbb Z.
 \label{eq:alg-rounding-face}
\end{equation}
Use the finite row distributions of Lemma~\ref{lem:row-rounding},
independently across rows, as the virtual initial distribution.
The legal face is $P_r=\{z\in[0,1]^{S_r}:\sum_{k\in S_r}z_k=R_r\}$;
coordinates outside $S_r$ are fixed at their floors, so a zero fractional
coordinate is never ceiled.

For forecasts $a,b$, define the four-state cost, including all zero-mass
cases, by
\[
 H_k(Z,O):=
 \begin{cases}
 Z|a-O/(Z+O)|+O|b-O/(Z+O)|,&Z,O>0,\\
 Z|a|,&O=0<Z,\\
 O|b-1|,&Z=0<O,\\0,&Z=O=0.
 \end{cases}
\]
Thus a one-sided mode has destination $0$ or $1$, and $0/0$ is never formed.
For
\begin{equation}
 C_k[s,t]:=H_k(\ell_{0k}+s,\ell_{1k}+t),\qquad s,t\in\{0,1\},
 \label{eq:alg-four-state-cost}
\end{equation}
omit the impossible side of a one-sided mode.  If the current row is the
zero side, let $\pi_k$ be the conditional probability that the one side is
ceiled: it is $0$ or $1$ once that row is fixed and otherwise is $f_{1k}$.
The current coefficient is
\begin{equation}
 d_k^{(0)}=(1-\pi_k)(C_k[1,0]-C_k[0,0])
 +\pi_k(C_k[1,1]-C_k[0,1]).
 \label{eq:alg-rounding-coeff-zero}
\end{equation}
If the current row is the one side, $\pi_k$ is instead the conditional
probability that the zero side is ceiled: it is $0$ or $1$ after that row is
fixed and otherwise is $f_{0k}$.  Then
\begin{equation}
 d_k^{(1)}=(1-\pi_k)(C_k[0,1]-C_k[0,0])
 +\pi_k(C_k[1,1]-C_k[1,0]).
 \label{eq:alg-rounding-coeff-one}
\end{equation}
The one-sided formula is the difference of its two single-sided values.
Fixing a row means replacing its law by a point mass, while leaving the
remaining row laws independent; no zero-probability event is conditioned on.
The resulting expectation is affine on $P_r$; sort only $S_r$ by these rational
coefficients and ceil the $R_r$ smallest coordinates.  This chooses a
minimizing vertex without materializing the product distribution.

Every source total is preserved and every nonempty mode receives prediction
$O/(Z+O)$, or its endpoint in the one-sided case.  For
$q_k=v_k/(u_k+v_k)$, the triangle inequality gives
\[
 H_k(Z,O)\leq Z|a-q_k|+O|b-q_k|
 +|q_k(Z-u_k)-(1-q_k)(O-v_k)|.
\]
Each positive fractional coordinate has expected absolute error
$2f(1-f)\leq1/2$, including one-sided modes by the endpoint formula.  Hence
the expected total is at most $F'+J/2$, and the deterministic conditional
choices attain no larger cost.  Empty modes are discarded before evaluating
$O/(Z+O)$; equal final predictions may then be merged.
\end{proof}

\subsection{Procedure, Output, and Complexity}

Delete zero rows.  If $I_1=\varnothing$, assign every zero-label identity to
$0$; if $I_0=\varnothing$, assign every one-label identity to $1$.  In
either case $B=U=L$ and the endpoint records are returned.  Otherwise:
\begin{enumerate}
\item construct $K_P,K_D$ and the exact oracles of Lemma~\ref{lem:alg-bodies},
      with the dual box precheck before its continuous slack scan;
\item invoke GLS with the tolerances in \eqref{eq:alg-gls-body-deltas};
\item perform the exact primal row and $Q$-grid repairs and the exact dual
      slack shift, obtaining $B,U$; retain the repaired primal flow and
      shifted dual potentials as their rational feasibility witnesses;
\item lift, compress by Lemma~\ref{lem:alg-compression}, and round by
      Lemma~\ref{lem:alg-rounding};
\item output every nonempty mode as its two source identifiers (or a null
      endpoint marker), nonnegative integers $(Z,O)$ with $Z+O>0$, and
      destination $O/(Z+O)$ or the appropriate endpoint.
\end{enumerate}
The compact list has at most $2m$ records and size
$O(m(\log(m+1)+\log(T+1)))$.  Its row totals equal the input multiplicities;
the accompanying primal and dual witnesses contain $O(m^2)$ rational
entries of polynomial bit length.  Expansion assigns exactly $T$
identities and costs $\Omega(T)$ record writes,
or $O(T(\log(m+1)+\log(T+1)))$ additional bit operations.
Combining the three lemmas gives
\[
 C_{\mathrm{out}}\leq F'+J/2\leq U+m\leq L+m+\eta,
\]
which proves Theorem~\ref{thm:gridfree-algorithm}.  The named fixed
polynomials and all rational bit-length claims are inherited from the
algorithm-details appendix; no measured software performance or optimality
claim is made.

\section{Sharp Order of the Sparse Additive Term}
\label{sec:sparse-tightness}

The linear dependence on the number of original forecast values in
Theorem~\ref{thm:sparse-upper} is necessary when the sample size is
unrestricted.  This section gives a geometric-multiplicity construction;
all its forecasts lie close to the same rational center.  The proof
controls every calibrated group, including groups that mix all scales.
Define
\[
 G(m):=\sup_{\substack{T\geq1,\ x\in\{0,1\}^T,\ p\in[0,1]^T\\
                        |\{p_t:t\in[T]\}|\leq m}}
       \bigl(\CalDist(x,p)-\LowerCalDist(x,p)\bigr).
\]

\begin{lemma}[A counting criterion for exact deterministic cost]
\label{lem:counting-lock}
Let the empirical label mean of a binary unit-weight sequence be
$r=a/b\in(0,1)$ in lowest terms, and put $e_t=|p_t-r|$.  If
\begin{equation}
 s\,\bigl|\{t:e_t>s\}\bigr|\leq\frac1{12b}
 \qquad\text{for every }s>0,
 \label{eq:counting-lock-condition}
\end{equation}
then $\CalDist(x,p)=\sum_t e_t$.
\end{lemma}

\begin{proof}
Assigning every sample to $r$ is calibrated and has cost $\sum_t e_t$.
For the opposite inequality, consider any nonempty calibrated group of
size $N$ with $u$ positive labels and destination $q=u/N$.  If $q=r$,
its cost equals the sum of its $e_t$.  Otherwise set $d=|q-r|$.
Since $bu-aN$ is a nonzero integer,
\[
 d=\frac{|bu-aN|}{bN}\geq\frac1{bN}.
\]
Let $g$ be the number of members of this group with $e_t>d/4$.
The global counting assumption at $s=d/4$ implies
\[
 g\leq\frac1{3bd}\leq\frac N3.
\]
For each such member the reverse triangle inequality gives
$|p_t-q|\geq e_t-d$.  For each of the remaining $N-g$ members it gives
\[
 |p_t-q|\geq d-e_t\geq e_t+d/2.
\]
Consequently the group cost is at least
\[
 \sum_{t\text{ in group}}e_t-dg+\frac d2(N-g)
 =\sum_{t\text{ in group}}e_t+\frac d2(N-3g)
 \geq\sum_{t\text{ in group}}e_t.
\]
This argument allows arbitrary members from different forecast values.
Summing over all groups proves the asserted lower bound.
\end{proof}

\begin{theorem}[Order-tight sparse comparison]
\label{thm:sparse-order-tight}
For every pair of integers $K,R\geq1$, there is a binary unit-weight
input with exactly $2K$ distinct rational forecasts and
\[
 T=\frac{4R(8^K-1)}7
\]
such that
\begin{equation}
 \CalDist=\frac K{32},\qquad
 0\leq\LowerCalDist\leq\frac{1-8^{-K}}{1792R},\qquad
 \CalDist-\LowerCalDist>\frac K{64}.
 \label{eq:geometric-sparse-gap}
\end{equation}
In particular, $G(1)=0$, and for every integer $m\geq2$,
\begin{equation}
 \frac{\lfloor m/2\rfloor}{32}\leq G(m)\leq m.
 \label{eq:sparse-order-tight}
\end{equation}
Thus $G(m)=\Theta(m)$.  For each fixed $K$, the lower-distance value in
\eqref{eq:geometric-sparse-gap} can be made arbitrarily small by increasing
$R$, while the deterministic distance remains $K/32$.
\end{theorem}

\begin{proof}
For $j=0,\ldots,K-1$, put
\[
 n_j=2R8^j,\qquad e_j=\frac1{64n_j},\qquad
 p_j^-=\frac12-e_j,\quad p_j^+=\frac12+e_j.
\]
At each of $p_j^-$ and $p_j^+$ place $n_j$ samples, exactly $n_j/2$
of each label.  All counts are integers, all forecasts lie in $(0,1)$,
and the $2K$ forecasts are distinct.  The sample size is
$2\sum_j n_j=4R(8^K-1)/7$, and the empirical label mean is $1/2$.

For any $s>0$, the indices with $e_j>s$ form an initial segment
$0,\ldots,h$, or the empty set.  In the nonempty case,
\[
 s\,\bigl|\{t:|p_t-1/2|>s\}\bigr|
 =2s\sum_{j=0}^h n_j
 <\frac{2}{64n_h}\frac{8n_h}{7}
 =\frac1{28}<\frac1{24}.
\]
The empty case is immediate.  Lemma~\ref{lem:counting-lock}, with
$r=1/2$ and $b=2$, therefore proves
\[
 \CalDist=\sum_{j=0}^{K-1}2n_j e_j=\frac K{32}.
\]
In particular, this equality already covers every cross-scale regrouping.

To bound the fractional distance, set $h=n_j e_j=1/64$ at every scale.
Move mass $h$ of one positive-label sample from $p_j^-$ to $p_j^+$,
and mass $h$ of one zero-label sample in the opposite direction.
Keep the remaining mass of these two samples, and every other sample,
at its original forecast.  Both selected labels are available because
$n_j/2\geq1$, and $0<h<1$ makes each source distribution valid.
The two destination masses remain $n_j$, and their positive-label masses
are $n_j/2-h$ and $n_j/2+h$.  Their means are exactly $p_j^-$ and $p_j^+$.
The cost at scale $j$ is $4n_j e_j^2$.  Hence this finite, exactly
calibrated witness gives
\[
 \LowerCalDist\leq\sum_{j=0}^{K-1}4n_j e_j^2
 =\frac1{1024}\sum_{j=0}^{K-1}\frac1{2R8^j}
 =\frac{1-8^{-K}}{1792R}.
\]
Since $K,R\geq1$ and $1/1792<1/64$, subtracting this bound from
$K/32$ gives the strict final inequality in
\eqref{eq:geometric-sparse-gap}.

For $m\geq2$, choose $K=\lfloor m/2\rfloor$; the construction has
at most $m$ forecasts for every $R$.  Letting $R$ increase without bound
in \eqref{eq:geometric-sparse-gap} gives $G(m)\geq K/32$, since the
definition of $G(m)$ does not restrict $T$.  The upper bound is
Theorem~\ref{thm:sparse-upper}.  Finally, if there is only one original
forecast $p_0$, every calibrated fractional transport has mean destination
$\bar x=T^{-1}\sum_t x_t$, so its cost is at least $T|p_0-\bar x|$.
The deterministic assignment of every sample to $\bar x$ attains that
cost.  Thus $G(1)=0$.
\end{proof}

The theorem proves the optimal \emph{order} of the sparse additive term;
it does not identify its best constant.  Its sample size is exponential
in $K$ when $R$ is fixed, so it does not give a matching lower bound
for every prescribed pair $(T,m)$.  The separate length-dependent exponent
is determined in Section~\ref{sec:matching-upper}.

\section{A Matching Length-Dependent Upper Bound}
\label{sec:matching-upper}

The length-dependent exponent can be matched by rounding a calibrated
measure on an adaptive rational grid.  The relevant size of the grid is
the sum of its reciprocal denominators, rather than its cardinality.
All measures considered in this section are nonnegative and finitely supported.
Throughout this section, $W_1$ denotes the unnormalized minimum transport
cost with ground distance $|q-r|$ for finite measures of equal mass.

\begin{lemma}[An adaptive rational grid]
\label{lem:adaptive-farey-grid}
For every integer $H\geq1$, there is a finite increasing grid
\[
 0=q_0<q_1<\cdots<q_J=1,\qquad q_i=a_i/b_i
\]
of reduced fractions such that
\begin{equation}
 q_{i+1}-q_i=\frac1{b_i b_{i+1}}\leq H^{-2},
 \qquad \sum_{i=0}^J\frac1{b_i}\leq5H.
 \label{eq:adaptive-grid-bounds}
\end{equation}
The endpoint denominators are $b_0=b_J=1$.
\end{lemma}

\begin{proof}
Start with $0/1,1/1$.  Whenever two adjacent fractions $a/u<c/v$
satisfy $uv<H^2$, insert their mediant $(a+c)/(u+v)$.
The adjacent determinant $cu-av=1$ is preserved by this operation;
in particular, the mediant is reduced.  Every inserted denominator
satisfies
\[
 u+v\leq uv+1\leq H^2.
\]
Only finitely many reduced fractions with these denominators lie in
$[0,1]$, and each insertion adds a new fraction.  Thus the process
terminates.  At termination the determinant identity gives the asserted
gap formula, and every adjacent denominator product is at least $H^2$.

An inserted fraction determines an ordered pair $(u,v)$ of parent
denominators with $uv<H^2$.  Conversely, at most one inserted fraction
can have a given ordered pair: $cu-av=1$, $0\leq a<u$, and $1\leq c\leq v$
determine the parent numerators uniquely.  Indeed, for $u>1$ the congruence
$av\equiv-1\pmod u$ determines $a$, since $u,v$ are coprime; the case
$u=1$ gives $a=0,c=1$.  Hence
\[
 \sum_{i=0}^J\frac1{b_i}
 \leq2+\sum_{\substack{u,v\geq1\\uv<H^2}}\frac1{u+v}.
\]
For $1\leq u\leq H$, monotonicity and integration give
\[
 \sum_{v=u}^{\lfloor H^2/u\rfloor}\frac1{u+v}
 \leq\frac1{2u}+\ln\!\left(\frac{u+H^2/u}{2u}\right)
 \leq\frac1{2u}+2\ln(H/u).
\]
Using symmetry to restrict the double sum to $u\leq v$, and allowing
extra boundary terms, bounds it by
\[
 \sum_{u=1}^H\frac1u+4\bigl(H\ln H-\ln(H!)\bigr)
 \leq H+4(H-1).
\]
Here $\ln(H!)\geq\int_1^H\ln t\,dt=H\ln H-H+1$.
Adding the two endpoints gives at most $5H-2\leq5H$.
\end{proof}

For a finite measure $\mu$ on forecast values, its calibrated label
measures are
\[
 \mu_0(\{q\})=(1-q)\mu(\{q\}),\qquad
 \mu_1(\{q\})=q\mu(\{q\}).
\]
Thus the mass of $\mu_1$ is the first moment of $\mu$.

\begin{lemma}[Transport between calibrated label measures]
\label{lem:calibrated-transport-lift}
If finite measures $\mu,\nu$ on $[0,1]$ have the same total mass and
the same first moment, then their calibrated label measures satisfy
\begin{equation}
 W_1(\mu_0,\nu_0)+W_1(\mu_1,\nu_1)\leq2W_1(\mu,\nu).
 \label{eq:calibrated-transport-lift}
\end{equation}
\end{lemma}

\begin{proof}
Let $\pi(q,r)$ be an optimal coupling of $\mu,\nu$.
For each pair $(q,r)$, match positive-label mass
$\min\{q,r\}\pi(q,r)$ and zero-label mass
$\min\{1-q,1-r\}\pi(q,r)$.  The cost of this common part is at most
$\sum_{q,r}|q-r|\pi(q,r)=W_1(\mu,\nu)$.
The remaining positive-label source and target masses have totals
\[
 A:=\sum_{q,r}(q-r)_+\pi(q,r),\qquad
 B:=\sum_{q,r}(r-q)_+\pi(q,r),
\]
respectively.  Equality of the first moments gives $A=B$.  The remaining
zero-label source and target masses have totals $B,A$, so they can also
be matched without changing labels.  Every remaining transport distance
is at most one.  These two matches cost at most
$A+B=\sum_{q,r}|q-r|\pi(q,r)=W_1(\mu,\nu)$.
If a remaining mass is zero, its matching is empty.  Adding the common
and remaining costs proves the claim.
\end{proof}

\begin{theorem}[Sharp order of the length-dependent gap]
\label{thm:matching-global-upper}
For every binary unit-weight input of length $T\geq1$, and every
integer $H\geq1$,
\begin{equation}
 \CalDist(x,p)\leq\LowerCalDist(x,p)+\frac{T}{H^2}+20H.
 \label{eq:global-H-bound}
\end{equation}
Consequently $\Delta_T\leq41T^{1/3}$ for all $T\geq1$.  Together with
Theorem~\ref{thm:growing-lower}, this gives
\begin{equation}
 \frac{T^{1/3}}{1000}\leq\Delta_T\leq41T^{1/3}
 \qquad(T\geq216),
 \label{eq:matching-global-gap}
\end{equation}
and hence $\Delta_T=\Theta(T^{1/3})$.
\end{theorem}

\begin{proof}
Write $S=\sum_t x_t\in\mathbb Z$.  Fix any finite calibrated fractional
witness of cost $F$.  Let $\mu$ be its total destination measure; its
mass is $T$ and its first moment is $S$.  We construct a finite measure
$\mu^*$ whose calibrated label masses are integers, whose total mass
and first moment are $T,S$, and for which
\begin{equation}
 W_1(\mu_0,\mu^*_0)+W_1(\mu_1,\mu^*_1)
 \leq\frac{T}{H^2}+20H.
 \label{eq:integer-calibrated-measure-approximation}
\end{equation}

Choose the grid of Lemma~\ref{lem:adaptive-farey-grid}.
For $q\in[q_i,q_{i+1}]$, split the mass at $q$ between the two endpoints
in proportions
\[
 \frac{q_{i+1}-q}{q_{i+1}-q_i},\qquad
 \frac{q-q_i}{q_{i+1}-q_i}.
\]
Grid points remain fixed.  This preserves both mass and first moment.
Its cost per unit mass is
\[
 \frac{2(q-q_i)(q_{i+1}-q)}{q_{i+1}-q_i}
 \leq\frac{q_{i+1}-q_i}{2}\leq\frac1{2H^2}.
\]
Let $M_i$ be the resulting mass at $q_i$, and write
$\mu^{\rm g}=\sum_i M_i\delta_{q_i}$.  Then
\begin{equation}
 W_1(\mu,\mu^{\rm g})\leq\frac{T}{2H^2},\qquad
 \sum_iM_i=T,\quad\sum_iq_iM_i=S.
 \label{eq:mean-preserving-grid-quantization}
\end{equation}

Now send remainders only to the right.  Starting with $c_{-1}=0$, put
\begin{equation}
 N_i=b_i\left\lfloor\frac{M_i+c_{i-1}}{b_i}\right\rfloor,
 \qquad c_i=M_i+c_{i-1}-N_i
 \quad(0\leq i\leq J).
 \label{eq:denominator-packet-rounding}
\end{equation}
Thus $N_i$ is a nonnegative integer multiple of $b_i$ and
$0\leq c_i<b_i$.  At the final endpoint, $b_J=1$ and
\[
 M_J+c_{J-1}=T-\sum_{i<J}N_i
\]
is an integer.  Therefore $c_J=0$ and $\sum_iN_i=T$.
The measure $\nu=\sum_iN_i\delta_{q_i}$ has integral calibrated label
counts: its positive count at $q_i$ is $k_i=a_iN_i/b_i$.

The rightward remainder coupling transports exactly $c_i$ across
$[q_i,q_{i+1}]$.  Its cost, and the corresponding increase in the first
moment, are both
\begin{equation}
 \delta:=\sum_{i<J}c_i(q_{i+1}-q_i)
        =\sum_i q_iN_i-S
        =S'-S,
 \qquad S':=\sum_i k_i\in\mathbb Z.
 \label{eq:packet-mean-error}
\end{equation}
The equality also follows directly by substituting
$N_i-M_i=c_{i-1}-c_i$ and telescoping.
In particular $\delta$ is a nonnegative integer.  The grid bound gives
\begin{equation}
 W_1(\mu^{\rm g},\nu)\leq\delta
 \leq\sum_{i<J}b_i(q_{i+1}-q_i)
 =\sum_{i<J}\frac1{b_{i+1}}\leq5H.
 \label{eq:packet-transport-bound}
\end{equation}

It remains to restore the prescribed positive-label total $S$.
This adjustment is made to the auxiliary target histogram, not to any
observed input label.  Choose integers $0\leq d_i\leq k_i$ with
$\sum_i d_i=\delta$; such a choice exists because
$0\leq\delta=S'-S\leq S'=\sum_i k_i$.
For every $N_i>0$, give that target group $k_i-d_i$ positive labels,
$N_i-k_i+d_i$ zero labels, and destination
\[
 q_i^*:=\frac{k_i-d_i}{N_i}.
\]
Discard zero-mass groups.  The resulting $\mu^*$ has integral calibrated
label counts, total mass $T$, and positive-label total $S$.
Moving each auxiliary group from $q_i$ to $q_i^*$ costs
\[
 \sum_{i:N_i>0}N_i|q_i-q_i^*|=\sum_i d_i=\delta.
\]
Hence
\[
 W_1(\mu,\mu^*)\leq\frac{T}{2H^2}+2\delta
                     \leq\frac{T}{2H^2}+10H.
\]
The two measures now have the same first moment, so
Lemma~\ref{lem:calibrated-transport-lift} proves
\eqref{eq:integer-calibrated-measure-approximation}.

Finally, the fractional witness already couples each original label class
to $\mu_0$ or $\mu_1$, with total cost $F$.  Compose with the preceding
label-preserving transports to the integral target counts of $\mu^*$.
There is an assignment of whole source identities with no greater cost:
for each label, list the source forecasts and the target forecasts with
their integer repetitions in increasing order and pair them.  This is
optimal for absolute distance, by the usual uncrossing inequality
\[
 |a-c|+|b-d|\leq|a-d|+|b-c|
 \qquad(a\leq b,\ c\leq d).
\]
Each target group receives exactly its prescribed integer label counts
and is therefore calibrated; coincident destinations may be merged.
No original label is changed.  The total cost is at most
$F+T/H^2+20H$.

For every $\varepsilon>0$, take a finite witness with
$F<\LowerCalDist+\varepsilon$ and then let $\varepsilon$ decrease to
zero.  This proves \eqref{eq:global-H-bound} without assuming attainment
of the fractional infimum.  Taking $H=\lceil T^{1/3}\rceil$ and using
$H\leq2T^{1/3}$ for $T\geq1$ gives the asserted constant 41.
\end{proof}

\section{Discussion and Limitations}\label{sec:discussion}

The previous offline comparison had additive order $O(\sqrt T)$; the
sharp result here is $\Theta(T^{1/3})$ for the worst-case gap and
$\Theta(T^{-2/3})$ after normalization.  The gap is an integrality effect
from requiring one calibrated destination per observed identity, not
sampling uncertainty.  For sparse forecasts the normalized discrepancy is
at most $m/T$.  Neither statement says that every input has a nonzero gap.

For certificates $B\leq L\leq U$, let
$\rho:=\min\{m,41T^{1/3}\}$.  The two comparisons give
\[
 \frac CT\in\left[\max\!\left\{0,\frac BT\right\},
 \min\!\left\{1,\frac{U+\rho}{T}\right\}\right].
\]
If $U-B<\eta$, this interval has width at most
$(\rho+\eta)/T$.  An endpoint may certify a threshold decision for this
observed input, but an interval straddling the threshold does not.  The
repair produced by Theorem~\ref{thm:gridfree-algorithm} itself has cost
at most $U+m$; these intervals are not confidence intervals or
guarantees for future forecasts, decision regret, or truthful reporting.

The sharp orders have separate quantifiers: the length lower bound is
existential for each integer $T\geq216$, whereas the sparse lower bound
allows $T$ to grow without restriction.  The additional bound
$\Delta^{\leq}_{T,m}\leq36\sqrt m(1+\ln T)$ rules out a uniform
lower law of order $\min\{m-1,T^{1/3}\}$, but does not identify the
true joint order.  The comparisons concern binary unit-weight sequences,
not general weighted atoms or nonbinary outcomes.

Theorem~\ref{thm:gridfree-algorithm} assumes rational forecasts, rational
accuracy, and binary-encoded counts of separately assignable samples.  Its
compact output has at most $2m$ modes, but expansion requires
$\Omega(T)$ records.  The GLS reduction and exact post-processing prove a
bit-complexity theorem, not an implemented general optimizer, measured
performance, or an arbitrary-real-input algorithm.  The fixed offline
comparisons do not determine the optimal exponent for online forecasting.

\appendix
\section{Finite Rounding and Endpoint Details}\label{app:finite-rounding}

\begin{lemma}[Finite row rounding with prescribed marginals]\label{lem:row-rounding}
Let $w_1,\ldots,w_J\geq0$ have integer sum $n$.
There is a finite-support random vector $N\in\mathbb Z_{\geq0}^J$ such that
\[
  \sum_{k=1}^J N_k=n,
  \qquad
  \mathbb E N_k=w_k,
  \qquad
  w_k=0\Longrightarrow N_k=0.
\]
Moreover, if $w_k=\lfloor w_k\rfloor+f_k$ with $f_k\in[0,1)$, then
\[
  \mathbb E\abs{N_k-w_k}=2f_k(1-f_k)\leq\frac12.
\]
\end{lemma}

\begin{proof}
Write $w_k=\lfloor w_k\rfloor+f_k$; then $R=\sum_kf_k$ is an integer.
If $f$ is integral, take $B=f$.  Otherwise it has at least two
fractional coordinates $i,j$.
Set
\[
  \alpha:=\min\{1-f_i,f_j\}>0,
  \qquad
  \beta:=\min\{f_i,1-f_j\}>0,
\]
and replace $f$ by $f+\alpha(e_i-e_j)$ with probability
$\beta/(\alpha+\beta)$, or by $f-\beta(e_i-e_j)$ otherwise.
Both choices preserve the sum, stay in $[0,1]^J$, and make at least
one selected coordinate integral; their conditional mean is $f$.
After at most $J$ steps on each branch, the finite distribution is
supported on zero-one vectors $B$ with sum $R$ and mean $f$.
Take $N_k=\lfloor w_k\rfloor+B_k$.  This preserves zeros, and
\[
  \mathbb E\abs{N_k-w_k}
  =f_k(1-f_k)+(1-f_k)f_k
  =2f_k(1-f_k)\leq\frac12.
\]
\end{proof}

Independent copies may be used for the source rows, though
Section~\ref{sec:sparse-rounding} needs only their marginal identities.
At $q=0$ the calibration equation forces all one-label masses to zero;
at $q=1$ it forces all zero-label masses to zero.  The zero-preserving
rounding above therefore handles endpoint modes.  Pure-label inputs
and the use of near-optimal finite witnesses are treated in
Sections~\ref{sec:preliminaries} and~\ref{sec:sparse-rounding}.

\section{Geometry and Duality of the Rational Bodies}
\label{app:algorithm-geometry}

This appendix supplies the geometric details of
Lemma~\ref{lem:alg-bodies}.  Its numerical bit bounds are treated separately
in Appendix~\ref{app:algorithm-bits}.

\subsection{Primal Inner and Outer Balls}
Use the uniform feasible flow
$u^0_{ij}=n_i^0/|I_1|$, $v^0_{ij}=n_j^1/|I_0|$ and $h^0=T+1/2$.
Every flow entry is at least $1/m$.  Since
$0\leq g_{ab}(u,v)\leq u+v$, its objective $F(u^0,v^0)$ is at most $T$.
A Euclidean perturbation of radius $r_P=1/(100m^3)$ in the free
coordinates changes each retained coordinate by at most $r_P$ and each
eliminated coordinate by at most $mr_P$.  All reconstructed flows stay
positive.  Their total $\ell_1$ change is at most $2m^3r_P$, so
\eqref{eq:alg-g-lipschitz} changes $F$ by at most $1/10$.
The $h$ coordinate changes by at most $r_P$, less than the remaining
epigraph slack; it also remains in $(0,T+1)$.  This proves the full
Euclidean ball inclusion, including $h$.  There are at most $2m^2+1$
free coordinates, each bounded in magnitude by $T+1$, and hence
$R_P=4m(T+1)$ is an outer radius about the origin.

\subsection{Supporting Potentials and Strong Duality}
Allow the active marginal vector to vary over the positive orthant of
$\mathbb R^{I_0}\times\mathbb R^{I_1}$, and let $V(n)$ be the minimum
in \eqref{eq:alg-finite-representation} with that marginal vector.
The same perspective formulation is feasible and compact for each $n$;
it gives a finite convex, positively homogeneous value function.
Indeed, combining feasible flows proves convexity, and scaling flows in
both directions proves $V(tn)=tV(n)$ for $t>0$.  A finite convex function
has a supporting subgradient $\lambda=(\alpha,\beta)$ at the interior
point $n$.  Applying its inequality to $tn$ for $t>1$ and $t<1$ yields
$\lambda^{\mathsf T}n=V(n)$.

For any active pair and $q\in[0,1]$, let
$a_{ijq}=(1-q)e_i^{(0)}\oplus q e_j^{(1)}$.  Adding a mode of weight
$t>0$ to an optimal transport at $n$ gives
\[
 V(n)+t\lambda^{\mathsf T}a_{ijq}
 \leq V(n+t a_{ijq})\leq V(n)+t c_{ij}(q).
\]
Thus $(1-q)\alpha_i+q\beta_j\leq c_{ij}(q)$ for every pair and $q$.
The supporting potential is dual feasible and attains $V(n)=L$.
Conversely, integrating these inequalities over any paired transport
gives weak duality.  This proves \eqref{eq:alg-dual}; it does not infer
duality from the subsequent box or separation tests.

\subsection{The Bounded Active Dual Body}
Endpoint constraints give $\alpha_i\leq a_i\leq1$ and
$\beta_j\leq1-a_j\leq1$.  If an optimal active potential had, for
example, $\alpha_i<-T$, its objective would be strictly less than
\[
 -Tn_i^0+(T-n_i^0)=T-n_i^0(T+1)\leq-1,
\]
since $n_i^0\geq1$.  This contradicts $L\geq0$, the value attained or
improved upon by the all-zero feasible potential.  The same argument
applies to every active coordinate.  Consequently the larger box
$[-2T,1]$ retains an optimizer.  At the all-$(-1)$ vector, every continuous
constraint has slack at least one and normal of Euclidean norm at most
one.  The vector is at distance at least one from the box boundary, so
its radius-$1/2$ Euclidean ball lies in $K_D$.
Finally $\|y\|_2\leq2T\sqrt{2m}\leq4T(m+1)=R_D$ for every $y\in K_D$.

Extend a feasible active potential by $-1$ on deleted zero-count types.
For an inactive zero type and an active one type, the endpoint bound gives
\[
 -(1-q)+q\beta_j
 \leq-(1-q)+q(1-a_j)
 \leq q|a_j-q|\leq c_{ij}(q),
\]
where the middle inequality follows from
$|a_j-q|\geq q-a_j$ and $(1-q)^2\geq0$.
For an active zero type and an inactive one type, use
$|a_i-q|\geq a_i-q$ and $q^2\geq0$ instead.
If both types are inactive the left side is $-1$.
Thus the extension is feasible and leaves the objective unchanged.

\section{Bit Complexity of the Rational Reduction}
\label{app:algorithm-bits}

For inputs with both labels, this appendix proves the bounds used in
Lemmas~\ref{lem:alg-kernel-bits} and~\ref{lem:alg-oracle-encoding};
pure-label endpoints need no GLS call.  Their costs sum at most $m$
rational terms: the product of input denominators and the integer counts
give polynomial-size work registers, covered by the setup allowance below.
Let $S$ be as in
\eqref{eq:alg-composition-size}, $\eta_0=\min\{\eta,1\}$, and let
$\ell$ denote binary magnitude length including a sign bit.  Write
\[
 E=|I_0||I_1|\leq m^2,\qquad
 H_{\rm act}=|I_0|+|I_1|\leq2m,\qquad J_0\leq2m^2.
\]
Uniformly for $S\geq1$ and query length $q\geq0$, the fixed
oracles and two GLS calls give
\[
\begin{gathered}
 D\leq P_{\rm body}(S),\quad
 N_*\leq2P_{\rm calls}(\Lambda_S),\quad
 Q_{\max}\leq P_{\rm qry}(\Lambda_S),\\
 \ell(\text{each post-GLS work integer})\leq P_{\rm int}(S),
 \qquad \operatorname{Time}_{\rm compact}\leq P_{\rm steps}(S).
\end{gathered}
\]
Columnwise integerization controls compression; oracle encodings
and GLS transcripts bound the five post-GLS stages.  External GLS
coefficients remain symbolic.

\subsection{Kernel Construction and Integer Workspace}
\label{app:algorithm-kernel}

Let $\mathcal A=I_0\sqcup I_1$ and $r=|\mathcal A|=H_{\rm act}$.
If compression is needed, select $r+1$ positive modes.  For
$q_\ell=p_\ell/d_\ell$ in lowest terms, the column in the \emph{full}
active space is
\[
 A_{\bullet\ell}
 =\frac{d_\ell-p_\ell}{d_\ell}e_{i(\ell)}^{(0)}
  \mathbin{\oplus}\frac{p_\ell}{d_\ell}e_{j(\ell)}^{(1)}.
\]
A one-sided mode retains its zero coordinate.  Store all current
weights as $w_\ell=W_\ell/H_j$.  If $q_{\rm col}$ bounds column
numerators and denominators and $q_{\rm aux}$ bounds input,
accuracy, and pivot record fields, put
\begin{equation}
 h_0:=q_{\rm col}+q_{\rm aux}+64+
 \lceil\log_2(2m+1)\rceil+\lceil\log_2(T+1)\rceil+8.
 \label{eq:alg-kernel-bit-base}
\end{equation}
For $s\geq1$ define
\begin{equation}
 M_s(h_0):=s h_0+\lceil\log_2(s!)\rceil+2,\qquad
 R_{\rm B}(h_0):=2M_{2m}(h_0)+\lceil\log_2(2m+1)\rceil+8.
 \label{eq:alg-kernel-minor-bounds}
\end{equation}
If the selected rank is $\rho\leq r$, choose $\rho$ independent
active rows and $\rho+1$ columns of rank $\rho$; give other selected
columns coefficient zero.  Omitted rows lie in the chosen rows'
span on the selected columns, so the kernel solves the full active
system.  Integerize \emph{each chosen column separately}:
\[
 \widehat A_\ell
 =(d_\ell-p_\ell)e_{i(\ell)}^{(0)}
   \mathbin{\oplus}p_\ell e_{j(\ell)}^{(1)}.
\]
Each entry has at most $q_{\rm col}+1\leq h_0$ bits.  Leibniz
expansion gives $M_s(h_0)$ for every $s$-minor without multiplying
column denominators.  At Bareiss step $s<\rho$, the exact quotient
is an $(s+1)$-minor; its undivided difference of minor products
fits $R_{\rm B}(h_0)$.

Write the chosen matrix as $[C\ b]$ with $C$ nonsingular.  One
fraction-free elimination gives $U,g$ and $\Delta=\det C$.
Set $z_{\rho+1}=\Delta$ and back-substitute
\[
 z_i=\frac{-\Delta g_i-\sum_{j>i}U_{ij}z_j}{U_{ii}}.
\]
Cramer's rule identifies $z_t$, up to a common sign, with
$(-1)^t\det\widehat A_{\widehat t}$: divisions are exact and
$\ell(z_t)\leq M_\rho(h_0)$.  Each undivided sum of at most
$\rho+1$ minor products fits $R_{\rm B}(h_0)$; elimination and
back-substitution take cubic arithmetic time.

On the chosen columns set $d_\ell^{\rm ker}=d_\ell z_\ell$ and set it
to zero elsewhere.  Then $Ad^{\rm ker}=0$ and
$\ell(d_\ell^{\rm ker})\leq h_0+M_\rho(h_0)+2$.  Every source column
sums to one, so $\sum_\ell d_\ell^{\rm ker}=0$; the nonzero direction
has both signs.  Orient it by the exact cost sign and compare negative
ratios by cross multiplication.  For a minimizing index $k$,
\[
 t=\frac{W_k}{H_j|d_k^{\rm ker}|},\qquad
 H_{j+1}=H_j|d_k^{\rm ker}|,\qquad
 W'_\ell=W_\ell|d_k^{\rm ker}|+W_kd_\ell^{\rm ker}.
\]
The last update includes modes outside the kernel support.  If
integer payloads have at most $p_j$ bits, the ratio, comparison
products, new denominator, and all updated numerators satisfy
\begin{equation}
 \begin{gathered}
 K_s(h_0):=h_0+M_s(h_0)+\lceil\log_2(T+1)\rceil+8,\\
 \ell(\text{pivot registers})
 \leq p_j+K_\rho(h_0)\leq p_j+K_r(h_0).
 \end{gathered}
 \label{eq:alg-kernel-persistent-increment}
\end{equation}
In particular, $H_j$ gains the length of just one selected kernel
coordinate per pivot.  With
\[
 h_{\rm ker}:=h_0+\lceil\log_2((2m)!)\rceil+
                 \lceil\log_2(T+1)\rceil+20,
\]
$K_\rho(h_0)\leq(2m+2)h_{\rm ker}$ for $\rho\leq2m$.

\subsection{Oracle Encoding and Operation Counts}
\label{app:algorithm-encoding}

Let $Q$ be the complete query length.  Rational primitives reduce their
results to a signed numerator and positive denominator.  Depth counts
multiplication and division of rational values, not the integer products
inside rational addition.  For at most $A$ arithmetic gates, input
numerator and denominator payloads of at most $b\geq1$ bits, and depth
at most $\mu$, define
\begin{equation}
 \begin{aligned}
 \mathcal L_{-1}(A,b)&:=b,\\
 \mathcal L_k(A,b)&:=4(A+2)(\mathcal L_{k-1}(A,b)+2)
       \quad(0\leq k\leq\mu),\\
 \mathsf B(A,\mu,b)&:=8\mathcal L_\mu(A,b)+64.
 \end{aligned}
 \label{eq:alg-expression-bit-bound}
\end{equation}
This bounds a complete output record and temporary integer workspace,
including shared subexpressions, and is a fixed polynomial for fixed $\mu$.
For one output, cut the addition/subtraction subgraph at inputs and
multiplication/division outputs.  It has at most $A+1$ boundary values;
its integer coefficients have magnitude at most $2^A$.
If boundary payloads have $h$ bits, a common denominator uses at most
$(A+1)h+1$ bits, and the numerator at most
$(A+2)h+A+\lceil\log_2(A+1)\rceil+4$.
At depth zero use $h=b$; at depth $k\geq1$, a product or quotient of
lower-depth values has $h\leq2\mathcal L_{k-1}+2$.
Both bounds are below $\mathcal L_k$; gcd reduction cannot increase them.
Elias-gamma coding of the two payload lengths and the 64 fixed control
bits gives at most $8h+64$ record bits.  Unreduced products, sums and
comparison cross-products require at most $4\mathcal L_\mu+8$ integer
bits, also covered by $\mathsf B$.  Long common-denominator products and
compression iterations are bounded separately below.  Set
\begin{equation}
 \begin{aligned}
 b(S,Q)&:=S+Q+1,& A_{\rm pair}&:=256,& C_{\rm pair}&:=128,\\
 A_{\rm query}(m)&:=264m^2+8m+40,&
 C_{\rm query}(m)&:=134m^2+4m+16.
 \end{aligned}
 \label{eq:alg-local-operation-counts}
\end{equation}
The reconstructed-coordinate and atomic-record bounds are
\begin{equation}
 b_{\rm rec}(S,Q):=\mathsf B(2S+4,1,b(S,Q)),\qquad
 B_{\rm atom}(S,Q):=\mathsf B(A_{\rm query}(S),10,b_{\rm rec}(S,Q)).
 \label{eq:alg-atomic-bit-bound}
\end{equation}

The fixed program's per-query schedule is as follows.  Here
$d_D=H_{\rm act}$, and ``norm.'' denotes an exact normalization.
\begin{center}
\begin{tabular}{@{}lrrrrr@{}}
branch & candidates & arith. & cmp. & norm. & depth\\
\hline
primal pair & 3 & 48 & 16 & 0 & 8\\
dual pair & 8 & 240 & 96 & 16 & 8\\
primal pair aggregate & -- & 64 & 32 & 0 & 2\\
one final aggregate scan & -- & $8E+32$ & $6E+16$ & 0 & 2\\
dual box precheck & -- & $4d_D+8$ & $2d_D$ & 0 & 1
\end{tabular}
\end{center}
The primal row evaluates three pieces and a support.  The dual row
checks at most four interval endpoints and one stationary candidate
per interval, selects the minimum, and normalizes $(1-q,q)$.
With aggregation, a primal pair uses at most $48+64=112$ arithmetic
and $16+32=48$ comparisons; a dual pair uses $240+16=256$ and $96$.
Thus $A_{\rm pair}=256,C_{\rm pair}=128$ cover both branches.

The dual box precheck uses $2d_D\leq4m$ comparisons and
$4d_D+8\leq8m+8$ arithmetic/record operations: at most two form
$-2T$, and $2d_D+6$ form the signed unit vector, right-hand side,
and flags.  Its cut has coefficients in $\{-1,0,1\}$ and right-hand
side $1$ or $2T$; the complete-cut bound includes it.

For a primal query, process pairs $e=(i,j)$ in lexicographic order.
From the selected value $g_e$ and support $(s_e^u,s_e^v)$ at
$(u_e^0,v_e^0)$, maintain sparse support slots, row/column
accumulators, and the epigraph constant:
\[
 \kappa_e=g_e-s_e^u u_e^0-s_e^v v_e^0,\qquad
 F\mathrel{+}=g_e,\quad K\mathrel{+}=\kappa_e,\quad
 R_i\mathrel{+}=s_e^u,\quad C_j\mathrel{+}=s_e^v.
\]
Each update takes two products, two subtractions, four additions,
and two sparse writes; $64/32$ also covers lookup, flags, signs,
and records.  Reconstruct the $H_{\rm act}$ eliminated coordinates
in at most $2E$ operations, subtract pivot supports from the other
$2E-H_{\rm act}$ coefficients, and update at most $2H_{\rm act}$
pivot constants.  The final scan of
$d_{\rm cut}=2E-H_{\rm act}+1$ coordinates tests
$h^0<F(u^0,v^0)$ and returns one normalized cut within $8E+32$
arithmetic and $6E+16$ comparisons.  The multiplicative depth
remains at most ten; charging even a dual box scan followed by all
pairs gives
\begin{equation}
 \begin{aligned}
 A_{\rm query}(m)&\leq256m^2+(8m^2+32)+(8m+8)
                    =264m^2+8m+40,\\
 C_{\rm query}(m)&\leq128m^2+(6m^2+16)+4m
                    =134m^2+4m+16.
 \end{aligned}
 \label{eq:alg-query-count-derived}
\end{equation}

A pulled-back coefficient combines stored supports and at most
$H_{\rm act}$ pivot corrections.  Its three-level expression has at
most $3S^2+8S+32$ nodes, giving
\begin{equation}
 B_{\rm cut}(S,Q):=\mathsf B(3S^2+8S+32,3,B_{\rm atom}(S,Q)).
 \label{eq:alg-cut-bit-bound}
\end{equation}
Either complete cut has at most $2m^2+1\leq2S^2+1$
coordinates.  Charge every coordinate, separator, cut sign,
epigraph side, and exact-membership flag:
\begin{equation}
 P_{\rm sep}(S,Q):=(2S^2+2)(B_{\rm cut}(S,Q)+8)+16.
 \label{eq:alg-fixed-oracle-polynomial}
\end{equation}
For every $S\geq1$ and query length $q\geq0$, the \emph{same}
answer-length polynomial applies to both strong oracles:
\begin{equation}
 \phi_S(q):=P_{\rm sep}(S,q)
          =(2S^2+2)(B_{\rm cut}(S,q)+8)+16.
 \label{eq:alg-gls-answer-function}
\end{equation}
It bounds the complete answer or membership flag:
\begin{equation}
 Q_{\rm ans}\leq\phi_S(Q_{\rm qry}).
 \label{eq:alg-oracle-answer-bound}
\end{equation}

The finite name contains the fixed oracle program, input, masks,
elimination maps, radii, objectives, and tolerances.  Summing nonnegative
counts first constructs $T$ with payload length at most $S$.
All other fields use fixed formulas with at most 32 arithmetic gates
and depth eight on records bounded by $8S+64$.  In particular,
\[
 \delta_P=\frac{\eta_0}{200m^3(1+800m^4(T+1))},\qquad
 \delta_D=\frac{\eta_0}{8(T+1)(1+16T^2(m+1))}.
\]
These formulas need at most 16 gates; the affine coefficients and radii
need fewer.  Charge $8m^2$ primal, $8m^2$ dual, and $16m+32$
row/mask/header records.  For the once-fixed program length $\ell_{\rm prog}$,
\begin{equation}
 \begin{aligned}
 H_{\rm body}(S)&:=\mathsf B(32,8,8S+64),\\
 N_{\rm rec}(S)&:=16S^2+16S+32,\\
 P_{\rm body}(S)&:=\ell_{\rm prog}
                   +N_{\rm rec}(S)(H_{\rm body}(S)+8)+32S.
 \end{aligned}
 \label{eq:alg-body-description-bound}
\end{equation}
Both bodies' tolerances and radii are counted; the program checks
dual breakpoints and stationary points rather than listing
infinitely many constraints.

For the bit interface, prefix each request by its body name, including
a verbatim encoded input tuple; $Q$ counts this full message.
If the name has length $D$, then $D\geq N_\eta$.
Each of $m$, $\lceil\log_2(T+1)\rceil$ and
$\lceil\log_2(1/\eta_0)\rceil$ is at most $N_\eta$; for the count sum,
use $\prod_{i,y}(n_i^y+1)\geq1+T$.  Hence $S\leq5D\leq5Q$.
The single fixed polynomial $\Psi(q):=P_{\rm sep}(5(q+1),q)$ therefore
bounds all answers, as required by GLS~\S1.2.
Prefix copying adds only $D$ bits and fixed delimiters per call; it is
absorbed into the query-length polynomial and charged in communication below.

Use $\Psi$ for the GLS general assumption, retaining $\phi_S$ for
the instancewise accounting.  For
description length $D$, dimension $d$, radii $R,r$, objective
accuracy $\varepsilon$, and weak-separation parameter $\delta$, put
\begin{equation}
 \Lambda:=1+D+d+\lceil\log_2(1+R)\rceil+
 \lceil\log_2(1+1/r)\rceil+
 \lceil\log_2(1+1/\varepsilon)\rceil+
 \lceil\log_2(1+1/\delta)\rceil.
 \label{eq:alg-gls-call-interface}
\end{equation}
After fixing this protocol and $\Psi$, GLS~\S1.2 and Corollary~4.2.7
allow fixed symbolic constants $c_C,c_Q,c_W>0$ and degrees
$d_C,d_Q,d_W$ to be chosen uniformly over $S$,
both bodies, and queries; no numerical values are assigned:
\begin{equation}
 \begin{aligned}
 N_{\rm GLS}&\leq P_{\rm calls}(\Lambda):=c_C(1+\Lambda)^{d_C},\\
 Q_{\rm max}&\leq P_{\rm qry}(\Lambda):=c_Q(1+\Lambda)^{d_Q},\\
 N_{\rm ext}&\leq P_{\rm ext}(\Lambda):=c_W(1+\Lambda)^{d_W}.
 \end{aligned}
 \label{eq:alg-gls-external-polynomials}
\end{equation}
Here $N_{\rm ext}$ excludes communication and local oracle work.
The constants may be enlarged once to cover the prefixed-query wrapper.
Every generated WSEP query obeys
\begin{equation}
 Q_{\rm max}\leq c_Q(1+\Lambda)^{d_Q}.
 \label{eq:alg-gls-query-interface}
\end{equation}
For either body, $D\leq P_{\rm body}(S)$ and $d\leq2m^2+1$.
The recorded radii and tolerances contribute at most their
numerator and denominator lengths plus two each; thus both calls
are bounded by
\begin{equation}
 \Lambda_S:=1+4P_{\rm body}(S)+(2m^2+1)+8S+
       8\left\lceil\log_2(1+1/\eta_0)\right\rceil.
 \label{eq:alg-gls-composed-bound}
\end{equation}
For the two calls, $N_*\leq2P_{\rm calls}(\Lambda_S)$ and
$Q_{\max}\leq P_{\rm qry}(\Lambda_S)$.
The reference bit cost $\beta_{\rm bit}$ is justified below.
Charge setup once and work-table initialization at every oracle call:
\begin{equation}
 \begin{aligned}
 A_{\rm oracle}(S)&:=A_{\rm query}(S)+C_{\rm query}(S)+8S^2+32,\\
 A_{\rm setup}(S)&:=128N_{\rm rec}(S)+128(S+1)^2+256(S+1),\\
 A_{\rm or}^{\#}(S)&:=A_{\rm oracle}(S)+128(S+1)^2,\\
 b_{\rm or}(S,Q)&:=P_{\rm body}(S)+S+Q+\phi_S(Q)+16,\\
 P_{\rm setup}(S)&:=A_{\rm setup}(S)\beta_{\rm bit}(P_{\rm body}(S)+S+16),\\
 \operatorname{Time}_{\rm oracle}(S,Q)
   &:=A_{\rm or}^{\#}(S)\beta_{\rm bit}(b_{\rm or}(S,Q)).
 \end{aligned}
 \label{eq:alg-oracle-computation-time}
\end{equation}
Charge external transitions, query writes, full-answer reads,
and actual oracle computation separately:
\begin{equation}
 \begin{aligned}
 N_{\rm GLS\text{-}steps}:=\mathcal T_{\rm GLS}(S)
 &=2P_{\rm ext}(\Lambda_S)
   +N_*(Q_{\max}+\phi_S(Q_{\max})+1)\\
 &\quad+N_*\operatorname{Time}_{\rm oracle}(S,Q_{\max}).
 \end{aligned}
 \label{eq:alg-gls-step-interface}
\end{equation}
GLS output-record writes are included.  In setup, $128N_{\rm rec}$
pays for at most 32 arithmetic gates per field plus comparisons,
copying and indexing; $128(S+1)^2$ pays for local slots,
and $256(S+1)$ for input scans, counts, masks, and the fixed parameter
formulas.  Stationary-point evaluations depend on the query and belong
to the oracle, not setup.

After compression there are $J\leq H_{\rm act}$ modes.  A pivot
removes a positive mode, so
$p_{\rm comp}\leq J_0-J\leq2m^2$.  If $S_s$ is a source row's
fractional-support list, each paired mode meets at most two rows:
\begin{equation}
 I_{\rm inc}:=\sum_s|S_s|\leq2J\leq4m.
 \label{eq:alg-rounding-incidence-count}
\end{equation}
The repair, shift, compression, sorting, and rounding primitive
envelopes are
\begin{equation}
 \begin{aligned}
 A_{\rm repair}(m)&:=64m^2+64m+64,\\
 A_{\rm shift}(m)&:=1024m^2+128m+128,\\
 A_{\rm comp}(m)&:=4096m^5+512m^2+128,\\
 A_{\rm sort}(m)&:=256m\lceil\log_2(m+1)\rceil+256m,\\
 A_{\rm round}(m)&:=2048m^2+256m+128,\\
 A_{\rm post}(S)&:=A_{\rm repair}(S)+A_{\rm shift}(S)
                  +A_{\rm comp}(S)+A_{\rm sort}(S)+A_{\rm round}(S).
 \end{aligned}
 \label{eq:alg-post-gls-operation-counts}
\end{equation}
In expanded form,
\begin{equation}
 A_{\rm post}(m)=4096m^5+3648m^2+
 256m\lceil\log_2(m+1)\rceil+704m+448.
 \label{eq:alg-post-gls-operation-counts-expanded}
\end{equation}
Here are the primitive totals as
$(+/- ,\times,\text{exact/normalization division},
\text{comparison/record})$; sorting uses its last entry for
comparisons:
\[
\begin{aligned}
 \mathbf v_{\rm repair}
   &=(16,8,8,32)(m^2+m+1),\\
 \mathbf v_{\rm shift}
   &=32(8m^2+m+1,\;8m^2+m+1,\\
   &\hspace{17mm}4m^2+m+1,\;12m^2+m+1),\\
 \mathbf v_{\rm comp}
   &=(1024m^5+128m^2+32)(1,1,1,1),\\
 \mathbf v_{\rm sort}
   &=(0,0,0,A_{\rm sort}(m)),\\
 \mathbf v_{\rm round}
   &=(512m^2+64m+32,\;256m^2+32m+16,\\
   &\hspace{17mm}512m^2+64m+32,\;768m^2+96m+48).
\end{aligned}
\]
The local charges generating these totals are:
\begin{center}
\begin{tabular}{@{}lll@{}}
class & count at most & $(+/- ,\times,/,\text{tests/copies})$\\
\hline
repair flow & $2E$ & $(8,4,4,16)$\\
repair row & $H_{\rm act}$ & $(8,4,4,8)$\\
repair final & $1$ & $(16,8,8,16)$\\
shift candidate & $8E$ & $(32,32,16,48)$\\
shift potential & $H_{\rm act}$ & $(16,16,16,16)$\\
shift minimum/final & $1$ & $(32,32,32,32)$\\
rounding state & $4J$ & $(8,4,2,8)$\\
fractional support & $I_{\rm inc}$ & $(8,2,2,8)$\\
conditional coefficient & $I_{\rm inc}$ & $(8,4,2,8)$\\
final mode & $J$ & $(8,4,2,8)$\\
row bookkeeping & $H_{\rm act}$ & $(4,2,2,4)$
\end{tabular}
\end{center}
In particular, repair has at most $2m^2+2m+1$ objects,
and shift at most $8m^2+4m+1$.  Balanced mergesort uses at most
$k_s\lceil\log_2(k_s+1)\rceil$ comparisons on a row list of size
$k_s$; $I_{\rm inc}\leq4m$ places their sum within $A_{\rm sort}$.
Empty states are set to zero, without a $0/0$ division.
Returning the primal and dual witnesses and $B,U$ adds at most
$2m^2+2m+2$ scalar-record copies.  The rounding table's record charges
are at most $152m$; their sum with these copies is below
$768m^2+96m+48$, so $A_{\rm round}$ already covers certificate output.

For compression retain at most $J_0$ modes and set
$n=H_{\rm act}+1\leq2m+1\leq3m$, using the \emph{full}
active source count.  The per-pivot schedule is
\begin{center}
\begin{tabular}{@{}lrrrr@{}}
operation & $+/-$ & $\times$ & exact / & tests/copies\\
\hline
two elimination sweeps & $2n^3$ & $4n^3$ & $2n^3$ & $8n^3$\\
back substitution & $n^2$ & $n^2$ & $n$ & $4n^2$\\
scale and cost sign & $2n$ & $2n$ & $0$ & $4n+2$\\
minimum ratio and step & $n$ & $2n+1$ & $1$ & $4n+2$\\
weights and denominator & $n$ & $J_0+n+1$ & $0$ & $2J_0+4n+2$\\
active-list bookkeeping & $0$ & $0$ & $0$ & $2J_0+4n+2$
\end{tabular}
\end{center}
The sweeps select independent rows/columns and triangularize one
augmented square system; each uses at most $n^3$ trailing updates
(two products, a subtraction, an exact division).  Back-substitution
uses at most $n^2$ summands and $n$ quotients; the $G$-denominator
cost sign uses at most $n$ products/additions.  Ratios take at most
$2(n-1)$ cross-products and $n-1$ comparisons; forming $t$ takes
one product and quotient.  Updating \emph{all} numerators takes
$J_0$ scaling products, at most $n$ supported products/additions,
and one denominator product.  Searches and list maintenance are
charged to tests/copies.

Summing the table by column gives
\[
\begin{split}
 \mathbf v_{\rm pivot}\leq(&
 2n^3+2n^2+8n+8+4J_0,\quad
 4n^3+2n^2+8n+8+2J_0,\\
 &2n^3+n^2+4n+4,\quad
 8n^3+8n^2+16n+16+4J_0)\\
 &\leq(112,162,79,360)m^3
   <(512,512,512,512)m^3.
\end{split}
\]
One-time pair lifting, fixed-cost evaluation, and
shared-denominator preparation cost at most
$(16,16,8,16)$, $(8,8,4,8)$, and $(0,4,2,4)$
per relevant pair or mode.  Initialization and final scanning use
at most $8J_0+8$ record operations.  Each preparation class fits
$128m^2+32$, whence
\begin{equation}
 4(2m^2\cdot512m^3+128m^2+32)
 =A_{\rm comp}(m)=4096m^5+512m^2+128.
 \label{eq:alg-compression-count-closure}
\end{equation}

Signed numerator, positive denominator, sign/zero flag, two
separators, operation tag, and normalization/reduction flags incur
$8+16+16+24$ fixed record bits:
\begin{equation}
 C_{\rm rat}:=64.
 \label{eq:alg-rational-record-overhead}
\end{equation}
For payloads of at most $h$ bits, Elias-gamma length fields give at most
$8h+64$ record bits.  For numerical transition bounds, fix a 64-tape
binary machine, with one simultaneous read/write/head move per tape
per transition.  It stores current query/answer and the two solution
vectors, overwritten by their feasible repairs, not
entire GLS transcripts or virtual rounding distributions.  A record census
for bodies, flows/cuts, modes, elimination, rounding, and other scratch is
\[
\begin{aligned}
 N_{\rm live}\leq{}&N_{\rm rec}(S)+(24m^2+16m+16)+(20m^2+8m+16)\\
 &+(16m^2+24m+16)+(32m+16)+(8m^2+16m+64)\\
 \leq{}&84S^2+112S+160\leq128(S+1)^2.
\end{aligned}
\]
Modes have nine fields and an active index; the matrix allowance covers
$4n(n+1)$ positions for $n=H_{\rm act}+1\leq2m+1$.
Store one position-indexed sequence of variable-length records.  For
$b\geq\max\{S,64\}$ bounding all records and integer work registers,
lookup or write-back, including an in-place suffix shift of at most $b+2$
cells, costs at most $32N_{\rm live}(b+2)^2$ transitions.
This pays for scans, index comparisons and rewinds; a binary primitive
uses at most eight table accesses, with its arithmetic on fixed scratch tapes.

Finite-state carry/borrow scans give budgets $64(b+2)$ for integer
addition, comparison or copying, $128(b+2)^2$ for schoolbook multiplication,
and $256(b+2)^2$ for restoring division.  Division performs at most $b+1$
quotient-bit rounds of shifts, comparisons, subtraction and rewinds.
Euclid uses at most $2b+2$ divisions, since its larger remainder halves
within two rounds.  A rational primitive uses at most four products,
four divisions, one gcd and eight scans.  Including table access, its
coefficient after bounding all terms by $(b+2)^4$ is at most
\[
 8\cdot32\cdot128+4\cdot128+4\cdot256+512+8\cdot64
 =35328<2^{16}.
\]
Thus this reference implementation has the explicit primitive cost
$\beta_{\rm bit}(b):=2^{16}(b+2)^4$; this is a conservative proof bound,
not a software-performance claim.

These bound work registers, not just reduced outputs.  The interface from
the two GLS calls to the post-GLS arithmetic is fixed before any
$P_{\rm int}$ quantity is defined.  Let
\begin{equation}
 B_{\rm seed}(S):=P_{\rm body}(S)+S+Q_{\rm max}
       +\phi_S(Q_{\rm max})+b_{\rm or}(S,Q_{\rm max})+64.
 \label{eq:alg-gls-seed-storage}
\end{equation}
This bounds the live body description, request, answer, and local oracle
work registers at the start of either GLS call, and also dominates the
one-time setup records.  We use a fixed binary multi-tape reference machine
in which one transition writes at most one symbol on each of its 64 tapes.
If a run has $\tau$ transitions, then at most $64\tau$ bit writes and
$64(\tau+1)$ newly visited cells occur.  A persistent register is a
contiguous tape record, so its length is at most
$B_{\rm seed}+64(\tau+1)$.  Applying this to both calls gives
\begin{equation}
 Q_{\rm tm}(S):=B_{\rm seed}(S)+64\bigl(\mathcal T_{\rm GLS}(S)+1\bigr).
 \label{eq:alg-gls-tm-register-bound}
\end{equation}
Here $\mathcal T_{\rm GLS}$ is already defined in
\eqref{eq:alg-gls-step-interface} using only
$P_{\rm body},P_{\rm sep},P_{\rm calls},P_{\rm qry},P_{\rm ext}$ and
$\beta_{\rm bit}$; it contains no $P_{\rm int}$ or $P_{\rm steps}$.
Suppress $(S)$ on $Q_0,Q_{\rm data}$ below:
\begin{equation}
 \begin{aligned}
 Q_0(S)&:=\max\{Q_{\rm tm}(S),P_{\rm body}(S),
                 Q_{\rm max},\phi_S(Q_{\rm max}),\\
 &\hspace{12mm}S+\lceil\log_2(1+T)\rceil+
                    \lceil\log_2(1+1/\eta_0)\rceil+16\},\\
 Q_{\rm data}(S)&:=S+
      \left\lceil\log_2\left(\frac{40m^2}{\eta_0}+2\right)\right\rceil+16,\\
 q_{\rm rep}&:=2^{16}(A_{\rm repair}(m)+1)
                              (Q_0+Q_{\rm data}+66),\\
 \lambda_T&:=\lceil\log_2(T+1)\rceil,\qquad
 \lambda_m:=\lceil\log_2((2m)!)\rceil,\\
 X&:=q_{\rm rep}+Q_{\rm data}+64,\qquad
 b_{\rm loc}:=\mathsf B(64,8,X),\\
 q_{\rm shift}&:=8(2m^2+2m+4)(X+b_{\rm loc}+\lambda_T+16)+64,\\
 h_*&:=q_{\rm shift}+Q_{\rm data}+64+
       \lceil\log_2(2m+1)\rceil+2\lambda_T+\lambda_m+28,\\
 p_0&:=(2m^2+2)(q_{\rm shift}+64)+R_{\rm B}(h_*)+\lambda_T+16,\\
 q_{{\rm comp},0}&:=8p_0+64,\\
 p_{j+1}&:=p_j+(2m+2)h_*+8,\\
 q_{{\rm comp},j+1}&:=8p_{j+1}+64
 =q_{{\rm comp},j}+8(2m+2)h_*+64\quad(0\leq j<2m^2),\\
 q_{\rm round}&:=\mathsf B(A_{\rm round}(m)+A_{\rm sort}(m),16,
                     q_{{\rm comp},2m^2}+Q_{\rm data}+64),\\
 P_{\rm int}(S)&:=Q_{\rm post}(S):=
 q_{\rm round}+Q_{\rm data}+C_{\rm rat}
                     +\lceil\log_2(2m^2+2)\rceil.
 \end{aligned}
 \label{eq:alg-post-gls-bit-bound}
\end{equation}
In particular,
$q_{{\rm comp},2m^2}=q_{{\rm comp},0}+2m^2[8(2m+2)h_*+64]$.
The following constructions establish the denominators used in
this recurrence.

\subsection{Stage-by-Stage Denominators}
\label{app:algorithm-denominators}

The table lists actual pre-reduction denominators.  Product lengths
come from their factors, not a fixed depth.
\begin{center}
\begin{tabular}{@{}p{.15\linewidth}p{.62\linewidth}p{.15\linewidth}@{}}
stage & denominator constructions & register bound\\
\hline
primal repair &
 $D_r=\prod_jd_j$ for reconstruction and row normalization;
 $M_r$ for $n_rA_j/M_r$; lattice denominator
 $Q=\lceil40m^2/\eta_0\rceil$ & $q_{\rm rep}$\\
dual shift &
 selected slack denominator $e$; updated potentials have $d_i e$;
 certificate evaluation uses $e\prod_i d_i$ & $q_{\rm shift}$\\
compression &
 separate $H_0=\prod_k B_k$ and $G=\prod_k D_k$;
 then $H_{j+1}=H_j|d_k^{\rm ker}|$ & $q_{{\rm comp},j}$\\
sorting &
 incident masses have $H_jd_k$; state costs have
 $d_ad_bM$; coefficients use $e$ and four state-cost
 denominators & $q_{\rm round}$\\
rounding &
 the preceding mass, state, and coefficient denominators;
 a nonempty destination uses $Z+O$ & $P_{\rm int}(S)$
\end{tabular}
\end{center}

\subsubsection{Primal Repair}
\begin{lemma}[Explicit primal repair map]
\label{lem:alg-primal-repair-map}
Let $\bar y_P\in S(K_P,\delta_P)$ be the rational output of the primal
WOPT call.  The following rational map produces a feasible flow with an
epigraph value below $L+21\eta_0/50$ before lattice rounding.
\end{lemma}

\begin{proof}
Reconstruct $(\bar u,\bar v)$ from the free coordinates.  For analysis only,
choose $(u^*,v^*,h^*)\in K_P$ with
$\|\bar y_P-(u^*,v^*,h^*)\|_2\leq\delta_P$.  Reconstruction and
Cauchy--Schwarz give
\begin{equation}
 \|\bar u-u^*\|_1+\|\bar v-v^*\|_1
 \leq2m^3\delta_P\leq2m^3\varepsilon_P.
 \label{eq:alg-primal-reconstruction-distance}
\end{equation}
Set $\bar u^+_{ij}=\max\{\bar u_{ij},0\}$ and
$\bar v^+_{ij}=\max\{\bar v_{ij},0\}$, and define
\begin{equation}
 M_i^u=\sum_j\bar u^+_{ij},\quad
 \widehat u_{ij}=n_i^0\frac{\bar u^+_{ij}}{M_i^u},\qquad
 M_j^v=\sum_i\bar v^+_{ij},\quad
 \widehat v_{ij}=n_j^1\frac{\bar v^+_{ij}}{M_j^v}.
 \label{eq:alg-explicit-primal-repair}
\end{equation}
Positive-part truncation is nonexpansive.  Since each target row has mass at
least one and the total distance in \eqref{eq:alg-primal-reconstruction-distance}
is at most $2\eta_0/100<1/2$, every $M_i^u,M_j^v$ is positive.  If
$x\geq0$ has target mass $n\geq1$, $\|x-x^*\|_1=d<1/2$, and
$S=\sum x$, then $|S-n|\leq d$ and
$\|nx/S-x^*\|_1\leq(n/S)d+|n-S|n/S\leq4d$.
Therefore
\begin{equation}
 \|\widehat u-u^*\|_1+\|\widehat v-v^*\|_1
 \leq8m^3\varepsilon_P.
 \label{eq:alg-repaired-flow-distance}
\end{equation}
Set $\widehat h=F(\widehat u,\widehat v)$.  Then
$0\leq\widehat h\leq T<T+1$, and the Lipschitz bound gives
\begin{equation}
 \widehat h\leq h^*+40m^3\varepsilon_P
 \leq h(\bar y_P)+(1+40m^3)\varepsilon_P
 <L+(3/2+40m^3)\varepsilon_P
 \leq L+21\eta_0/50.
 \label{eq:alg-primal-repair-objective}
\end{equation}
All operations are sign tests, denominator products, exact divisions, and
row normalizations, so the resulting map is rational and its work records
are covered by the bounds below.
\end{proof}

An eliminated coordinate is a row count minus at most $m$ GLS
coordinates.  Put
\[
 B_{\rm rep}:=Q_0+Q_{\rm data}+64,\qquad
 b:=2(2m+5)(B_{\rm rep}+2).
\]
Partial denominator products take at most the sum of factor lengths
plus one bit.  Numerator sums add $\lceil\log_2(m+1)\rceil+1$
bits and the row count at most $\lambda_T+1$; thus $b$ covers
reconstructed and work integers.

For a reconstructed row $a_j/d_j$, $d_j>0$, let
\[
 D_r=\prod_{j=1}^{k_r}d_j,\qquad
 A_j=\max\{0,a_j\}(D_r/d_j),\qquad
 M_r=\sum_j A_j,\qquad k_r\leq m.
\]
Truncation keeps $d_j$ or uses denominator one; the row sum is
$M_r/D_r$ and normalization $n_rA_j/M_r$.  Since the prescribed
row count is at least one and reconstruction differs from a
nonnegative feasible row by at most $2m^3\varepsilon_P<1$,
we have $M_r>0$.
For partial as well as final products,
\[
 \ell(D_r)\leq\sum_{j=1}^{k_r}\ell(d_j)+1\leq k_rb+1,\quad
 \ell(A_j)\leq(k_r+1)b+2,\quad
 \ell(M_r)\leq(k_r+1)b+\lceil\log_2(k_r+1)\rceil+3.
\]
Construct $D_r$ with $k_r-1$ products and the $A_j$ with $k_r$
exact divisions/products.  Multiplication by $n_r$ adds at most
$\lambda_T+1$ bits.  Put
$\lambda_Q=\lceil\log_2(Q+1)\rceil\leq Q_{\rm data}$.
Lattice repair floors $Qn_rA_j/M_r$ and adds zero or one to
recover the row sum, with denominator $Q$ before reduction.
Set
\[
 H_{\rm rep}:=(m+1)b+\lambda_T+\lambda_Q+
                     \lceil\log_2(m+1)\rceil+8.
\]
This bounds $D_r,A_j,M_r,n_rA_j,Qn_rA_j$ and carries.
Division/remainder registers use at most $2H_{\rm rep}+2$ bits,
and complete records/registers at most $16H_{\rm rep}+80$.
Consequently
\begin{equation}
 \begin{split}
 16H_{\rm rep}+80
 &\leq128(m+2)^2(B_{\rm rep}+2)\\
 &\leq2^{16}(A_{\rm repair}(m)+1)(B_{\rm rep}+2)=q_{\rm rep}.
 \end{split}
 \label{eq:alg-repair-product-bit-bound}
\end{equation}

\subsubsection{Dual Shift and Certificate Evaluation}
Each endpoint or stationary-point slack candidate uses at most
64 arithmetic gates and depth eight:
$b_\gamma:=\mathsf B(64,8,Q_0+Q_{\rm data}+64)$.
The exact minimum selects one candidate without multiplying their
denominators.  For $\gamma=g/e$ with $e>0$, shifting $A_i/d_i$ gives
\[
 \frac{A_ie-gd_i}{d_ie}.
\]
Its denominator has at most $\ell(d_i)+\ell(e)+1$ bits and its
numerator at most
$\ell(A_i)+\ell(e)+\ell(g)+\ell(d_i)+2$ bits.
Certificate evaluation forms $e\prod_i d_i$ over at most $2m$
potentials.  Partial products fit $\ell(e)+\sum_i\ell(d_i)+1$
bits; the shift ledger charges at most $2m-1$ products and $2m$
exact quotients plus the weighted sum.  Each slack candidate and each
lifted mode's weight and cost use at most 64 gates and depth eight,
so $b_{\rm loc}$ bounds their records.  Compute $U$ by a common
denominator over at most $2m^2$ mode costs and $B$ over at most $2m$
weighted potentials.  The factor $2m^2+2m+4$ in $q_{\rm shift}$ pays
for these products, the input lengths $X$, the count factors
$\lambda_T$, and summation carries before the $8h+64$ record charge.

\subsubsection{Compression}
For at most $J_0$ modes write $w_k=A_k/B_k,c_k=C_k/D_k$,
initially bounded by $q_{\rm shift}$.  Prepare \emph{separate}
common denominators:
\[
 H_0=\prod_{k=1}^{J_0}B_k,\quad
 G=\prod_{k=1}^{J_0}D_k,\quad
 W_k=A_k(H_0/B_k),\quad
 \widetilde C_k=C_k(G/D_k).
\]
Preparation takes at most $2(J_0-1)$ denominator products and
$2J_0$ exact quotients and numerator products.  Partial and final
denominators fit $J_0q_{\rm shift}+1$ bits; numerators fit
$(J_0+1)q_{\rm shift}+2$.  The payload bound $p_0$ covers these
integers and the independent $R_{\rm B}(h_*)$ elimination workspace.
Test the sign with $\sum_k\widetilde C_kd_k^{\rm ker}$ while weights
retain $H_j$.

Apply the preceding kernel proof with
$q_{\rm col}\leq q_{\rm shift}$ and
$q_{\rm aux}\leq Q_{\rm data}$.  Then
$h_0\leq h_{\rm ker}\leq h_*$ and
$R_{\rm B}(h_0)\leq R_{\rm B}(h_*)$.
Inductively, every integer $W_\ell,H_j,G,\widetilde C_\ell$ has
payload length at most $p_j$, and its complete rational record has
at most $8p_j+64$ bits.  The kernel update adds at most
$(2m+2)h_*+8$ payload bits; $\lambda_T+8$ covers count factors,
ratio products and the supported direction's sign sum.
Encoding only after this integer bound gives the stated $q_{{\rm comp},j}$
recurrence, over at most $2m^2$ pivots.

\subsubsection{Sorting}
For $q_k=p_k/d_k$, incident masses are
\[
 \frac{W_k(d_k-p_k)}{H_jd_k}
 \quad\text{or}\quad
 \frac{W_kp_k}{H_jd_k}.
\]
Their floors and fractional remainders retain this denominator;
each needs one exact integer division and subtraction.
For fixed state counts $Z,O$, set $M=Z+O$.  If $M>0$ and
$a=A/d_a$, $b=B/d_b$, the four-state cost is
\[
 H_k(Z,O)=
 \frac{Z|AM-d_aO|d_b+O|BM-d_bO|d_a}{d_ad_bM};
\]
if $M=0$, set the cost to zero without division.  Formal states
have $M\leq T+2$, including those excluded later by zero
fractional coordinates.  Hence their records have at most
$b_C:=8(2S+2\lambda_T+16)+64$ bits.  For conditional probability
$\pi=f/e$, a coefficient has the form
\[
 (1-\pi)(C_{10}-C_{00})+\pi(C_{11}-C_{01}).
\]
Its denominator divides $e$ times the four state-cost
denominators, with length at most $\ell(e)+4b_C$
plus fixed sign/separator fields; the one-sided formula is
covered too.  Sorting compares by cross multiplication, at most
doubling a coefficient payload without creating a persistent
denominator.

\subsubsection{Deterministic Rounding}
Rounding uses the preceding denominators; a nonempty mode outputs
$O/(Z+O)$ or its one-sided endpoint.  No fixed arithmetic depth is
assumed for products over a row or mode list.  If the factors have
payload lengths $b_1,\ldots,b_k$, every sequential or balanced partial
product has length at most $1+\sum_{j=1}^k b_j$; the corresponding
exact divisions, comparison products, and carries are charged by the
same factor-sum bound.  The finite state formulas and their constant-size
candidate evaluations fit the $q_{\rm round}$ envelope, while the
factor-sum and pivot bounds show that $P_{\rm int}(S)$ covers every
post-GLS numerator, denominator, quotient workspace, comparison
product, and compact record.  The operation envelopes yield
\begin{equation}
 \begin{aligned}
 P_{\rm steps}(S):={}&P_{\rm setup}(S)+\mathcal T_{\rm GLS}(S)\\
 &+A_{\rm post}(S)\beta_{\rm bit}(P_{\rm int}(S)).
 \end{aligned}
 \label{eq:alg-local-step-bound}
\end{equation}
The terms pay for one-time setup, both external GLS calls with full
communication and actual oracle work, and all five exact post-GLS stages.
This proves a fixed polynomial bit bound.  Its external GLS coefficients
and the once-fixed program-description length have not been assigned
unverified numerical values.

\section{Joint Parameter Bounds and Counterexamples}
\label{app:joint-bounds}

This appendix records the additional fixed-parameter result summarized in
Section~\ref{sec:preliminaries}.  It gives an upper bound depending on both
the number of original forecasts and the sample size, and proves that the
two separately sharp orders do not combine into a uniform matching lower
bound.  The exact joint order is not asserted.

Define
\[
 \Delta^{\leq}_{T,m}:=
 \sup_{\substack{x\in\{0,1\}^T,\ p\in[0,1]^T\\
                   |\{p_t:t\in[T]\}|\leq m}}
 \bigl(\CalDist(x,p)-\LowerCalDist(x,p)\bigr).
\]

\begin{lemma}[Occupied-interval rational grid]
\label{lem:joint-grid}
Let $J,T\geq1$ and let $z_1,\ldots,z_J\in[0,1]$.  There is a finite
increasing grid $0=q_0<\cdots<q_s=1$ of reduced fractions
$q_i=a_i/b_i$, with endpoint denominators one and adjacent determinant
$a_{i+1}b_i-a_ib_{i+1}=1$, such that every open grid interval containing
one of the $z_j$ has length at most $1/T$.  Moreover, with
$A=1+\ln T$,
\begin{equation}
 \sum_{i=0}^s\frac1{b_i}\leq6\sqrt J\,A.
 \label{eq:joint-grid-budget}
\end{equation}
\end{lemma}

\begin{proof}
Put $H=\lfloor\sqrt J\rfloor$ and start from $0/1,1/1$.  First insert
mediants while adjacent denominators $u,v$ satisfy $u+v\leq H$.
The determinant-one identity is preserved, so every mediant is reduced;
the process is finite and every remaining adjacent denominator sum is
greater than $H$.  There are at most $b-1$ reduced fractions in $[0,1]$
with denominator $b$, hence the initial reciprocal sum is at most
$2+\sum_{b=2}^{H}(b-1)/b\leq H+1$.

Next insert a mediant only in an interval containing one of the $z_j$ and
whose endpoint denominator product is less than $T$.  The new denominator
is $u+v\leq uv+1\leq T$, while all adjacent denominator sums remain
greater than $H$.  Thus the process terminates and every occupied final
interval has length $1/(uv)\leq1/T$.

If an insertion has a parent denominator $b<H$, that parent was already
in the initial grid.  On either side of that parent, successive new
denominators form a subprogression $v+tb$; charging each side to its
parent gives a reciprocal contribution at most
$b^{-1}\sum_{t=1}^{T}t^{-1}\leq A/b$.  This class contributes at most
$2HA$.  Every other insertion has both parent denominators at least $H$.
Charge it to a point lying in its parent interval.  Along one refinement
path denominators increase by at least $H$ and stay at most $T$, so one
path contributes at most $A/H$.  There are at most $J$ paths.  Therefore
\[
 \sum_i b_i^{-1}\leq H+1+(2H+J/H)A\leq6\sqrt J\,A,
\]
including the case $J=1$ directly.
\end{proof}

\begin{lemma}[Integral calibrated target for sparse support]
\label{lem:joint-integral-target}
Let $\mu$ be a finite measure on $[0,1]$ with at most $J$ support points,
integer mass $T\geq1$, and integer first moment $S\in[0,T]$.  There is a
finite measure $\mu^*$ with the same mass and first moment, such that both
$\mu^*_0$ and $\mu^*_1$ have integer masses at every support point, and
\begin{equation}
 W_1(\mu_0,\mu^*_0)+W_1(\mu_1,\mu^*_1)
 \leq25\sqrt J\,(1+\ln T).
 \label{eq:joint-integral-target}
\end{equation}
\end{lemma}

\begin{proof}
Apply Lemma~\ref{lem:joint-grid}, write $q_i=a_i/b_i$, and split each
atom in an occupied interval between its endpoints in barycentric
proportions.  This preserves mass and first moment, and costs at most
$1/2$ because every occupied interval has length at most $1/T$.  Let
$\mu^{\rm g}=\sum_iM_i\delta_{q_i}$ denote the resulting measure.

Set $c_{-1}=0$ and define
\[
 N_i=b_i\left\lfloor\frac{M_i+c_{i-1}}{b_i}\right\rfloor,
 \qquad c_i=M_i+c_{i-1}-N_i.
\]
Then $0\leq c_i<b_i$.  Since $b_s=1$ and the mass available at the
right endpoint is an integer, $c_s=0$ and $\sum_iN_i=T$.  Moreover
$k_i=q_iN_i$ is an integer.  Rightward transport of the remainders has
cost
\[
 \delta:=\sum_{i<s}c_i(q_{i+1}-q_i)
       =\sum_iq_iN_i-S\in\mathbb Z_{\geq0},
\]
by telescoping $N_i-M_i=c_{i-1}-c_i$.  The determinant-one identity gives
\[
 \delta\leq\sum_{i<s}b_i(q_{i+1}-q_i)
       =\sum_{i<s}b_{i+1}^{-1}
       \leq6\sqrt J(1+\ln T).
\]

Choose integers $0\leq d_i\leq k_i$ with $\sum_i d_i=\delta$ and replace
the positive count $k_i$ by $k_i-d_i$, moving the target destination to
$(k_i-d_i)/N_i$ when $N_i>0$.  This restores the first moment and costs
exactly $\delta$.  Thus $W_1(\mu,\mu^*)\leq1/2+2\delta$.  The two
calibrated label measures have equal mass and first moment, so
Lemma~\ref{lem:calibrated-transport-lift} gives
\[
 W_1(\mu_0,\mu^*_0)+W_1(\mu_1,\mu^*_1)
 \leq1+4\delta\leq1+24\sqrt J(1+\ln T),
\]
which implies \eqref{eq:joint-integral-target}.
\end{proof}

\begin{theorem}[All-input joint upper bound]
\label{thm:joint-upper}
For every $T,m\geq1$ and every binary unit-weight input with at most $m$
distinct forecasts,
\begin{equation}
 \CalDist(x,p)-\LowerCalDist(x,p)
 \leq36\sqrt m\,(1+\ln T).
 \label{eq:joint-upper}
\end{equation}
Consequently, \eqref{eq:joint-upper-summary} holds.
\end{theorem}

\begin{proof}
Fix a finite fractional witness of cost $F$.  By the cost-nonincreasing
mode compression in Lemma~\ref{lem:mode-compression}, it may be replaced
by a witness with at most $J\leq2m$ destinations.  Its destination measure
$\mu$ has mass $T$ and integer first moment $S=\sum_tx_t$.  Apply
Lemma~\ref{lem:joint-integral-target} and compose the original
label transports with the resulting label-preserving transports.

For each label, order source forecasts and target slots and match them in
increasing order.  The uncrossing inequality
\[
 |a-c|+|b-d|\leq|a-d|+|b-c|\qquad(a\leq b,\ c\leq d)
\]
shows that this whole-identity matching is no more expensive than the
fractional coupling.  The target counts are integer and calibrated, so it
is a deterministic repair.  Therefore
\[
 \CalDist\leq F+25\sqrt{2m}(1+\ln T)
          \leq F+36\sqrt m(1+\ln T).
\]
Taking the infimum over finite witnesses proves the theorem.
\end{proof}

\begin{corollary}[The proposed uniform joint lower law is false]
\label{cor:joint-law-false}
There is no absolute constant $c>0$ such that, for all sufficiently large
pairs $(T,m)$ with $m\geq2$,
\[
 \Delta^{\leq}_{T,m}\geq c\min\{m-1,T^{1/3}\}.
\]
\end{corollary}

\begin{proof}
Set $T=n^3$ and $m=n+1$.  The proposed lower scale is $n$, while
Theorem~\ref{thm:joint-upper} gives
\[
 \frac{\Delta^{\leq}_{n^3,n+1}}{n}
 \leq\frac{36\sqrt{n+1}(1+3\ln n)}{n}\longrightarrow0.
\]
This contradicts any fixed positive $c$.
\end{proof}

An expanded multiscale lower family, including endpoint-padding invariance,
is supplied separately as supplementary material.  It is not needed for the
negative joint-law conclusion above.

{\raggedright\bibliography{references}}
\end{document}